\documentclass[preprint,12pt,sort&compress]{elsarticle}

\usepackage{amsmath}
\usepackage{amssymb}
\usepackage{amsthm}
\usepackage{mathrsfs}

\usepackage[compatibility=false]{caption}
\usepackage{subcaption}

\usepackage{enumitem}

\usepackage[retainorgcmds]{IEEEtrantools}

\usepackage{multirow}

\usepackage[usenames,dvipsnames]{xcolor}

\usepackage{tikz}

\usepackage{setspace}

\usepackage{pdfcomment}

\RequirePackage[linesnumbered,ruled,vlined]{algorithm2e}

\theoremstyle{definition}
\newtheorem{myDef}{Definition}

\theoremstyle{plain}
\newtheorem{myTheorem}{Theorem}

\theoremstyle{plain}

\theoremstyle{plain}
\newtheorem{myCorollary}{Corollary}

\theoremstyle{definition}

\theoremstyle{remark}
\newtheorem{myExample}{Example}

\newenvironment{myExampleCont1}{
   \addtocounter{myExample}{-1} \begin{myExample}[continued]}{
   \end{myExample}
   }

\newcommand*\circled[1]{\tikz[baseline=(char.base)]{
  \node[shape=circle,draw,inner sep=0pt] (char) {#1};}}

\newcommand\independent{\protect\mathpalette{\protect\independenT}{\perp}}
\def\independenT#1#2{\mathrel{\rlap{$#1#2$}\mkern2mu{#1#2}}}

\journal{arXiv}

\begin{document}
\begin{frontmatter}



\title{Properties of an $N$~Time-Slice Dynamic Chain Event Graph}

\author[label1]{Rodrigo A. Collazo\corref{cor1}}
\ead{Collazo@marinha.mil.br}

\author[label2]{Jim Q. Smith}
\ead{J.Q.Smith@warwick.ac.uk}

\address[label1]{Department of Systems Engineering, Naval Systems Analysis Centre, Rio de Janeiro 20091-000, Brazil }

\address[label2]{Department of Statistics, University of Warwick, Coventry CV4 7AL, United Kingdom }

\cortext[cor1]{Corresponding author}
 
\begin{abstract}
A Dynamic Chain Event Graph (DCEG) provides a rich tree-based framework for modelling a dynamic process with highly asymmetric developments. An $N$ Time-Slice DCEG (${N\text{T-DCEG}}$) is a useful subclass of the DCEG class that exhibits a specific type of periodicity in its supporting tree graph and embodies a time-homogeneity assumption. Here some desired properties of an ${N\text{T-DCEG}}$ is explored. In particular, we prove that the class of \text{$N$T-DCEGs} contains all discrete $N$~time-slice Dynamic Bayesian Networks as special cases. We also develop a method to distributively construct an ${N\text{T-DCEG}}$ model. By exploiting the topology of an ${N\text{T-DCEG}}$ graph, we show how to construct intrinsic random variables which exhibit context-specific independences that can then be checked by domain experts. We also show how an ${N\text{T-DCEG}}$ can be used to depict various structural and Granger causal hypotheses about a given process.  Our methods are illustrated throughout using examples of dynamic multivariate processes describing inmate radicalisation in a prison.  
\end{abstract}


\end{frontmatter}


%%%%%%%%%%%%%%%%%%%%%%%%%%%%%%%%%%%%%%%%%%%%%%%%%%%%%%%%%%%%%%%
\section{Introduction}
\label{sec:Introduction}

In many real-world settings it has became increasingly evident that describing
a process directly through components of a multivariate time series enables us
to obtain more accurate and well-calibrated models. A Dynamic Bayesian Network (DBN) \cite{Dean.Kanazawa.1989,Nicholson.1992,Kjaerulff.1992} is a widely used family of graphical model for representing and reasoning within dynamic systems whose progress is recorded over a discrete time intervals \cite{Dabrowski.Villiers.2015,Rubio.etal.2014,Marini.etal.2015,Li.etal.2014,Wu.etal.2015,Sun.Sun.2015,Khakzad.2015}. However, in some context a DBN model is not able to represent all structural information of the target process \cite{Poole.Zhang.2003}. This is particularly the case when the process is more naturally described by concatenations of unfolding events rather than by a product space of preassigned set of random variables. In other situations, a relevant statement corresponding to a conditioned variable cannot be directly incorporated into a DBN model using directed edges because it is valid only for a certain combinations of values assumed by the conditioning variables. In the literature, this type of statements is sometimes referred to context-specific information~\cite{Spiegelhalter.Lauritzen.1990,Boutilier.etal.1996}. 

To circumvent these issues, collections of networks and embellishments in the form of trees have been added to the DBN framework and computationally implemented using the object-oriented programming paradigm~\cite{Booch.2007}: for instance, see the developments on context-specific BNs~\cite{Boutilier.etal.1996, Poole.Zhang.2003, McAllester.etal.2008}, Bayesian Multinet~\cite{Geiger.Heckerman.1996}, Similarity Networks~\cite{Heckermann.1991} and Object-Oriented BNs~\cite{Koller.Pfeffer.1997,Bangso.Wullemin.2000}. However, such frameworks  focus on minimizing the computational cost associated with the propagation of information and model learning at the expense of the graphical expressiveness for decision makers and domain experts. An alternative approach is to adopt a graphical framework that is completely different from these based on direct acyclic graphs and can directly express context-specific information we refer to above.

Tree-based graphical models have been established as a key user-friendly method to translate domain hypotheses about a dynamic process into mathematical representations. Since its paths can be used to depict the various possible sequences of events a unit can experience over time, a tree provides a modeller with a flexible framework to accommodate asymmetric developments and context-specific structures. The Dynamic Chain Event Graph (DCEG)~\cite{Barclay.etal.2015, Collazo.Smith.2018a} is a particular type of infinite tree-based graphical model developed for discrete longitudinal data. Build upon a Chain Event Graph (CEGs)~\cite{Smith.Anderson.2008} a DCEG~\cite{Barclay.etal.2015} was originally envisaged to elicit simple semi-Markov processes defined as a very small number of discrete states. Recall that a CEG is supported by a finite tree and so appropriate to construct models in a non-dynamic context. In~\cite{Collazo.Smith.2018a} we rigorously advanced the foundation of DCEGs for specific classes of Markov process and defined a promising subclass called $N$~Time-Slice DCEG ($N$T-DCEG). Using objects and a special family of CEGs, we then presented a methodology to construct an $N$T-DCEG and to reason with it. Its links with Markov processes were also systematically explored. 

Being designed to model time-homogeneous Markov processes, an $N$\text{T}-DCEG has a graphical structure corresponding to a finite cyclic graph. This graph compactly encapsulates the supporting infinite tree initially used to describe the process. Here we show how an $N$T-DCEG retains many useful properties of a DBN whilst providing an expressive framework for various tasks associated with collaborative work, reasoning and Granger causal interpretation. In particular, we propose methods for distributed model construction and identification of random variables that drive the underlying stochastic process and are not directly defined by domain experts.

In this paper we derive and describe some of these properties. In Section~\ref{sec:Background} we briefly review the DBN and the $N\text{T-DCEG}$ models  before discussing some links between DCEGs and DBNs in Section~\ref{sec:SDCEG}. For example, we prove that the well used Two Time-Slice Dynamic Bayesian Network (2T-DBN) can always be expressed as a 2T-DCEG. Section~\ref{sec:Composite_Model} develops a method which uses the framework of the class of $N$T-DCEG to elicit the driven structure of a dynamic process from a team of experts working in parallel.  

In Section~\ref{sec:Constructiong_Random_Variable} we show how implicit conditional independence relationships encoded in an $N\text{T-DCEG}$ can be read from its representation using the graphical concept of a cut. This enables us to identify from the topology of an ${N\text{T-DCEG}}$ convenient sets of random variables - often not immediately apparent from the original description - whose relationship captures critical conditional independences embedded within the described process. Smith and Anderson \cite{Smith.Anderson.2008} have argued that the cuts in a CEG can be used as an alternative framework to answer queries corresponding to the Pearl's d-separation theorem \cite{Pearl.1988} in a BN. We show that these constructions naturally extend to an $N\text{T-DCEG}$. In Section~\ref{sec:local_global_independence}, we then explore the ideas of local, contemporaneous and stochastic independences. These have a strong link with notions of Granger noncausality \cite{Granger.1969}, also \cite{Hsiao.1982,Geweke.1984,Eichler.2007,Eichler.Didelez.2010}. We conclude the paper with a short discussion. 

%%%%%%%%%%%%%%%%%%%%%%%%%%%%%%%%%%%%%%%%%%%%%%%%%%%%%%%%%%%%%%%
\section{Background}
\label{sec:Background}

In this section we revisit the DBN and the DCEG models. 

\subsection {Dynamic Bayesian Networks}
\label{subsec:BN_DBN}

Let $\boldsymbol{{Z}}^{(m)}\!=\!(\mathcal{Z}_1,\ldots,\mathcal{Z}_m), m \leq n$, be the first $m$ variables of a sequence of random variables $\boldsymbol{{Z}}=(\mathcal{Z}_1,\ldots,\mathcal{Z}_n)$. Take a  directed acyclic graphic (DAG) $\mathbb{D}=(V,E)$ such that each vertex~$v_i$, $v_i \in V$, represents a variable $\mathcal{Z}_i$. Let ${pa(\mathcal{Z}_j)}=\{\mathcal{Z}_i \in \boldsymbol{\mathcal{Z}}^{(j-1)};(v_i,v_j) \in E\}$ denote the parent set of $\mathcal{Z}_j$ with respect to~$\mathbb{D}$ and $P_{\boldsymbol{Z}}$~denote the joint distribution of~$\boldsymbol{Z}$. We can now introduce the ordered Markov property that enables us to relate $P_{\boldsymbol{Z}}$ to the graphical topology of~$\mathbb{D}$. We then use it to formally define a BN model \cite{Duda.etal.1976,Pearl.1985,Pearl.1988}.

\begin{myDef}
	\label{def:bn_omp}
	The joint distribution~${P}_{\boldsymbol{Z}}$ satisfies the \textbf{ordered Markov property} (OMP) relative to a DAG~$\mathbb{D}$ if for every pair of non-adjacent vertices $v_i$ and $v_j$ in $V$, $i<j$, a variable $\mathcal{Z}_j$ is conditionally independent of a variable $\mathcal{Z}_i, i<j,$ given its parent set $pa(\mathcal{Z}_j)$. 
\end{myDef}

\begin{myDef}
	\label{def:bn}
	A \textbf{Bayesian Network} (BN) is a graphical model constituted by a sequence of random variables~$\boldsymbol{{Z}}$ and by a DAG $\mathbb{D}$ such that the joint distribution~${P}_{\boldsymbol{Z}}$ satisfies the ordered Markov property relative to $\mathbb{D}$.
\end{myDef}

In its most common formulation \cite{Dean.Kanazawa.1989,Kjaerulff.1992,Nicholson.1992} a Dynamic Bayesian Network (DBN)  models the temporal relationship among variables that are observed at regular time intervals. So henceforth in this paper we let $\boldsymbol{{Z}}(t)$  denote a sequence of random variables $\boldsymbol{{Z}}$ observed at time~$t$.  

Assume that a DAG ${\mathbb{D}(T)\!=\!(V(T),E(T)))}$ represents the conditional independence relationships between the components of~$\boldsymbol{{Z}}(T)$.
Now define the set of temporal edges $E_\dagger(T)$. These are edges from a vertex $v_i(t) \in V(t)$, ${t<T}$, to a vertex $v_j(T) \in V(T)$. So these represent relationships between variables in different time-slices.  Note that there might be a temporal edge $(v_i(t),v_i(T))$. This would depict the dependence of a variable $\mathcal{Z}_i$ at time $T$ on its value at any previous time~$t$, $t<T$. Inheriting the usual semantics of a BN two non-adjacent vertices $v_i(t)\in V(t)$ and $v_j(T) \in V(T)$, such that $t \leq T$ and, if $t=T$, $i<j$, then imply that $\mathcal{Z}_j(T)$ is conditionally independent of a variable $\mathcal{Z}_i(t)$ given its parent set $pa(\mathcal{Z}_j(T))$, where $pa(\mathcal{Z}_j(T)) \subseteq \cup_{k=0}^{T-1} \boldsymbol{{Z}}(k) \cup \boldsymbol{{Z}}^{(j-1)}(T)$.
Finally, a DBN for the first $T$ time-intervals consists of DAG ${\bar{\mathbb{D}}(T)\!=\!(\bar{V}{(T)},\bar{E}{(T)})}$, where $\bar{V}{(T)}\!=\!\cup_{t=0}^T V(t)$ and ${\bar{E}{(T)}\!=\!\cup_{t=0}^T (E(t) \cup E_\dagger(t))}$.

Without further assumptions, the specification of a DBN model is challenging because a different DAG $\mathbb{D}(t)$ and its corresponding temporal edge set needs to be defined for each time-slice~$t$. So for practical reasons two additional conditions are often hypothesised. The first of these is to assume a Markov condition of order~$N\!-\!1$. This demands that the values of a variable at time~$t$ depend only on the values of variables at the last ${N\!-\!1}$ previous and current intervals. The second common hypothesis is to assume that the process is time-homogeneous. 

These assumptions greatly simplifies the specification of the models.  This is because we only have to elicit $N$ conditional probabilities tables, one for each of the first ${N\!-\!1}$ time-slices and another for the succeeding time-slices. Therefore, to obtain a DBN we only need to define a limited number of DAGs and temporal edges: the DAGs $\mathbb{D}(t),t=0,\ldots,{N\!-\!2}$, for the first ${N\!-\!1}$ time-slices and their corresponding sets of temporal edges $E_\dagger(t)$; and a DAG $\mathbb{D}(t)\equiv \mathbb{D}, t={N\!-\!1},N,\ldots$, for all subsequent intervals and its corresponding set of temporal edges $E_\dagger(t)\equiv E_\text{\circled{$\dagger$}}$. When these two additional assumptions are adopted a DBN is called a $N$~Time-Slice DBN ($N$T-DBN). 

A common choice in practice is to set $N=2$, see e.g. \cite{Korb.Nicholson.2011,Neapolitan.2004,Pourret.etal.2008}. This implies that the current value of a given variable may persist in the system at maximum one time-slice ahead. In this case, the state of the system at time $t+1$ is completely determined by the values of its variables at time~$t$. Although these are strong assumptions, they nevertheless appear to provide satisfactory result particularly in systems that evolve slowly over time when we are interested in filtering and forecasting over short-term time horizon. 

\begin{myExample}
\label{ex:Radicalisation_Dynamic_3variables}

\emph{We revisit the example in~\cite{Collazo.Smith.2018a} that summarily describe the radicalisation dynamic of inmates in a prision using three random variables:
	\begin{itemize}[nosep]
		\item the variable called Network~($N$) distinguishes the following three levels of social contact of a ``standard" inmate with potential recruiters at each time interval~$t$: $\text{s- sporadic}$, $f$- frequent or $i$- intense; 
		\item the variable called Radicalisation~($R$) categorises a prisoner into one of the following states at each time interval~$t$: resilient to (r), vulnerable to (v) or adopting (a) radicalisation.
		\item the variable variable called Transfer~($T$) is a categorical variable indicating if a prisoner remains in prison ($n$) or was transferred~($t$) at the end of a time interval~$t$. 
	\end{itemize} 
	}

\emph{
	Assuming time-homogeneity from the onset and the 1-Markov condition, the set of conditional statements below fully characterizes our elicited radicalisation process:
	\begin{enumerate}[nosep]
		\item $T(0) \independent N(0) | R(0)$.
		\item $T(0) \independent N(0) | R(0) \neq a$.
		\item $R(t+1) \independent N(T) | (R(t), T(t)=n)$, for $t=0,1,\ldots$.
		\item $N(t+1) \independent R(T) | (N(t), T(t)=n)$, for $t=0,1,\ldots$.
		\item $T(t+1) \independent (N(T),R(t)) | T(t)=n$, for $t=0,1,\ldots$.
		\item $T(t+1) \independent N(T+1) | (R(t+1),T(t)=n)$, for $t=0,1,\ldots$.		
		\item $R(t+1) \independent N(t+1) | (R(t)=a,T(t)=n)$, for $t=0,1,\ldots$.		
		\item $R(t+1) \independent R(t) | (N(t+1)=a,R(t) \neq a, T(t)=n)$, for $t=0,1,\ldots$.
		\item $T(t+1) \independent N(T+1) | (R(t+1) \neq a,T(t)=n)$, for $t=0,1,\ldots$.			
	\end{enumerate}
	} 

\emph{
	A 2T-DBN corresponding to this description given that an inmate remains in prison until time~$t$ is depicted in Figure~\ref{fig:DBN_Radicalisation}. Note that a standard DBN is unable to depict the context-specific conditional statements elicited in statements 2, 7, 8 and 9. These relationships remains hidden within the conditional probability tables.}

%\vspace{-10pt}
\begin{figure}[ht]
\begin{center}
\includegraphics[scale=0.179,angle=-90,origin=c,trim=0 0 0 -190]{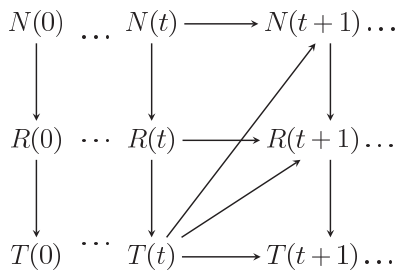}
\end{center}
\vspace{-30pt}
\caption{2T-DBN associated with Example \ref{ex:Radicalisation_Dynamic_3variables} \label{fig:DBN_Radicalisation}}
%\vspace{0pt}
\end{figure}

\end{myExample}

\subsection{Dynamic Chain Event Graphs}
\label{subsec:DCEG}

A DCEG model~\cite{Collazo.Smith.2018a} provides a compact colourful representation of a probabilistic tree through three simple steps: the elicitation of an event tree, its transformation into a staged tree and finally its wrapping into a DCEG graph. Using colours and some graphical transformations, this framework combines the graphical expressiveness of DBNs with the flexibility of tree graphs to depict both context-specific information and asymmetric developments.   

We first need to elicit an infinite event tree~$\mathcal{T}_\infty$~\cite{Shafer.1996} that describes qualitatively how the target process may unfold into sequences of events over discrete time intervals. Formally, an event tree~$\mathcal{T}$ is a rooted directed tree. Each leaf vertex~$l_i$ represents a possible terminating state of the process so has no emanating edges. A non-leaf vertex called situation~$s_i$, on the other hand, is a state from which a transition is possible. Both types of vertices are characterised by the chain of events that happens along either a root-to-$s_i$ path or a root-to-$l_i$ path. Each edge can be labelled by an event. Recall that in a graph~$\mathbb{G}=(V_\mathbb{G},E_\mathbb{G})$ a path is a subgraph~$\mathbb{G}_\mathbb{P}=(V_\mathbb{P},E_\mathbb{P})$, such that $V_\mathbb{P}=\{v_{i_1},\ldots,v_{i_L}\}$ and $E_\mathbb{P}=\{(v_{i_1},v_{i_2}),(v_{i_2},v_{i_3}),\ldots,(v_{i_{L-1}},v_{i_L})\}$, where  $V_\mathbb{P} \!\subseteq\! V_\mathbb{G}$, $E_\mathbb{P} \!\subseteq\! E_\mathbb{G}$ and all vertices~$v_{i_k}$, $k=1,\ldots,L$, are distinct~\cite{Cowell.etal.2007,Diestel.2006}. 

We next construct the staged tree~$\mathcal{ST}_\infty$ that is a coloured probability tree supported by the event tree~$\mathcal{T}_\infty$. To embed a probability map within~$\mathcal{T}_\infty$, each situation~$s_i$ is associated with a random variable~$X(s_i)$, whose state space~$\mathbb{X}(s_i)=\{\gamma_{ij}\}$ describes all immediate unfolding events~$\gamma_{ij}$ that may happen to unit at~$s_i$. Let $ch(v)=\{v' \in V_\mathbb{G}; (v,v') \in E_{\mathbb{G}}\}$ be the child set of a vertex~$v$ in any graph~$\mathbb{G}=(V_{\mathbb{G}},E_{\mathbb{G}})$. For each situation~$s_i$ in~$\mathcal{T}_\infty$, we can define the primitive probabilities
\vspace{-7pt}
$$\pi_{ij}=\pi(v_j|s_i)=P(X(s_i)=\gamma_{ij}|s_i), v_j \in ch(s_i).\vspace{-7pt}$$
Note that each event~$\gamma_{ij}$ labels the edge~$(s_i,v_j)$, $v_j \in ch(s_i)$. Two situations~$s_a$ and~$s_b$ in an event tree is said to be in the same \emph{stage~u} if and only if there is a one-to-one domain mapping between the state space of~$\mathbb{X}(s_a)$ and~$\mathbb{X}(s_b)$ and their corresponding primitive probabilities are the same. Associating each stage with a unique colour and then embellishing the vertices of an even tree with these colours, we obtain a staged tree.

Introducing the concept of a position enables us to identify and graphically suppress redundant conditional independent structures that our probabilistic model may present. Let~$\Lambda(\mathcal{T})$ be the set of paths of a tree~$\mathcal{T}$ until time~$t$. Also let~$\mathcal{T}(v)$ be the subtree that unfolds from a vertex~$v$ in a tree~$\mathcal{T}$. Two situations~$s_a$ and~$s_b$ in a staged tree~$\mathcal{ST}_\infty$ is said to be at the same \emph{position~w} if and only if there is a one-to-one mapping between~$\Lambda(\mathcal{ST}(s_a))$ and~$\Lambda(\mathcal{ST}(s_b))$ such that the following conditions hold:
\begin{enumerate}[nosep]
	\item \emph{Global condition} -  for any path~$\lambda$ in~$\Lambda(\mathcal{ST}(s_a))$ there is a path~$\lambda'$ in~$\Lambda(\mathcal{ST}(s_b))$ whose sequence of events and colours equals the sequence of events and colours of~$\lambda$.
	\item \emph{Local condition} - the numbers of events that happen during the $t^{th}$ time-slice associated with a path~$\lambda$ in~$\Lambda(\mathcal{ST}(s_a))$ and its corresponding path~$\lambda'$ in~~$\Lambda(\mathcal{ST}(s_b))$ are equal. 
\end{enumerate}

Finally, a \emph{DCEG graph}~$\mathbb{C}$ results from two simple graphical transformations of a staged tree: to merge all situations lying in the same position~$w$ into a single vertex~$w$ and to gather all leaf nodes, if they exist, into a single sink vertex~$w_\infty$. Let $\mathcal{F}(\mathbb{C})=\{\mathbb{C}_t; t=0,1,\ldots\}$ be a set of CEGs associated with a DCEG~$\mathbb{C}$, where $\mathbb{C}_t$ is the CEG supported by the staged tree corresponding to the first $t$-time-slices of~$\mathbb{C}$.

\subsubsection{An $N$ Time-Slice Dynamic Chain Event Graph}
\label{subsubsec:NTDCEG}

In~\cite{Collazo.Smith.2018a} we develop a method for depicting an infinite event tree using a set of objects, each of whom encapsulates a finite event corresponding to a finite process. The objects are concatenated according some domain-based rules. A common assumption is to define two objects, $\Delta(\mathcal{T}_{-1})$ and $\Delta(\mathcal{T})$. The object~$\Delta(\mathcal{T}_{-1})$ wraps a finite event tree~$\mathcal{T}_{-1}$ whose leaves define different type of unit observed in the process. Of course, if there does not exist domain information that enables us to distinguish the units, then $\mathcal{T}_{-1}$ is empty. In contrast, the object~$\Delta(\mathcal{T})$ represents a finite event tree that describe the possible sequences of events that might happen to a unit over a prescribed time interval. In this sense, the resulting infinite tree is periodic since we unfold the same object~$\Delta(\mathcal{T})$ from each situation representing a state at the beginning of each time interval. This type of infinite tree is said to be a \emph{Tree Object Generated by finite event trees~$\mathcal{T}_{-1}$ and~$\mathcal{T}$} (TOG(~$\mathcal{T}_{-1}$,$\mathcal{T}$)). 

Another useful way to define the graphical transformation of the staged tree is through the concept of {T-position}. This is a refinement of the concept of position in the sense that situations in the same T-position are necessarily in the same position but the converse does not always hold. Let $s(t)$ be a situation at time~$t$. Two situations $s_a(t_a)$ and $s_b(t_b)$ are said too be in the same \emph{${T}${-position}} if and only if they are in the same position, and one of the following conditions is valid: $t_a,t_b \in \{T,T+1,\ldots\}$ or $t_a=t_b=t$, $t \in\{-1,0,\ldots,T-1\}$. Using a $T$-position we can guarantee that a DCEG graph is acyclic before time~$T$. This is a useful property for a collaborative model construction: see Section~\ref{sec:Composite_Model}.

Take a infinite staged tree that is time-homogeneous after time~$N\!-\!1$ and whose supporting event tree can be described as a TOG($\mathcal{T}_{-1}$,$\mathcal{T}$). An \emph{\text{$N$T-DCEG} graph} is necessarily constructed from this type of staged tree, where:
\begin{itemize}[nosep]
	\item all situations in the same \text{$(N\!-\!1)$-position} is diverted into a single vertex;
	\item all leaf nodes at each time-slice~$t$ are represented by a single vertex~$w_\infty^t$, $t=-1,0,\ldots,N\!-\!2$; and
	\item all leaf nodes that unfolds from time-slice~$N\!-\!1$ on are collected by a single sink vertex~$w_\infty$.
\end{itemize}

The example below illustrates how to model the radicalisation process using the \text{$N$T-DCEG} framework.

\begin{myExampleCont1}
	\emph{
		Return to Example~\ref{ex:Radicalisation_Dynamic_3variables}. We can elicit the radicalisation process using the TOG($\mathcal{T}_{-1}$,$\mathcal{T}$) showed in Figure~\ref{fig:ObjectOriented_InfiniteEventTree_Radicalisation_3var}, where $\mathcal{T}_{-1}$ is empty and $\mathcal{T}$ is depicted in Figure~\ref{fig:FiniteEventTree_Radicalisation}. Colouring it according to the conditional independence statements we obtain the staged tree that is fully expressed by the $2$T-DCEG depicted in Figure~\ref{fig:2T-DCEG_Radicalisation_3var}. To construct this $2$T-DCEG model it is necessary to construct the staged tree corresponding to the first three time-slices and then to apply the graphical transformation rules presented above. For a more detailed discussion on $N$T-DCEG building, see~\cite{Collazo.Smith.2018a}.	
	}

\begin{figure}[h!]
	\centering
	%$\qquad \; \;$
	$\!\!\!\!\!\!\!\!\!\!\!\!\!\!\!\!\!\!\!$
	\begin{minipage}{.5\textwidth}
		\centering
		\includegraphics[scale=0.315,angle=-90,origin=c,trim=0 150 -100 0] {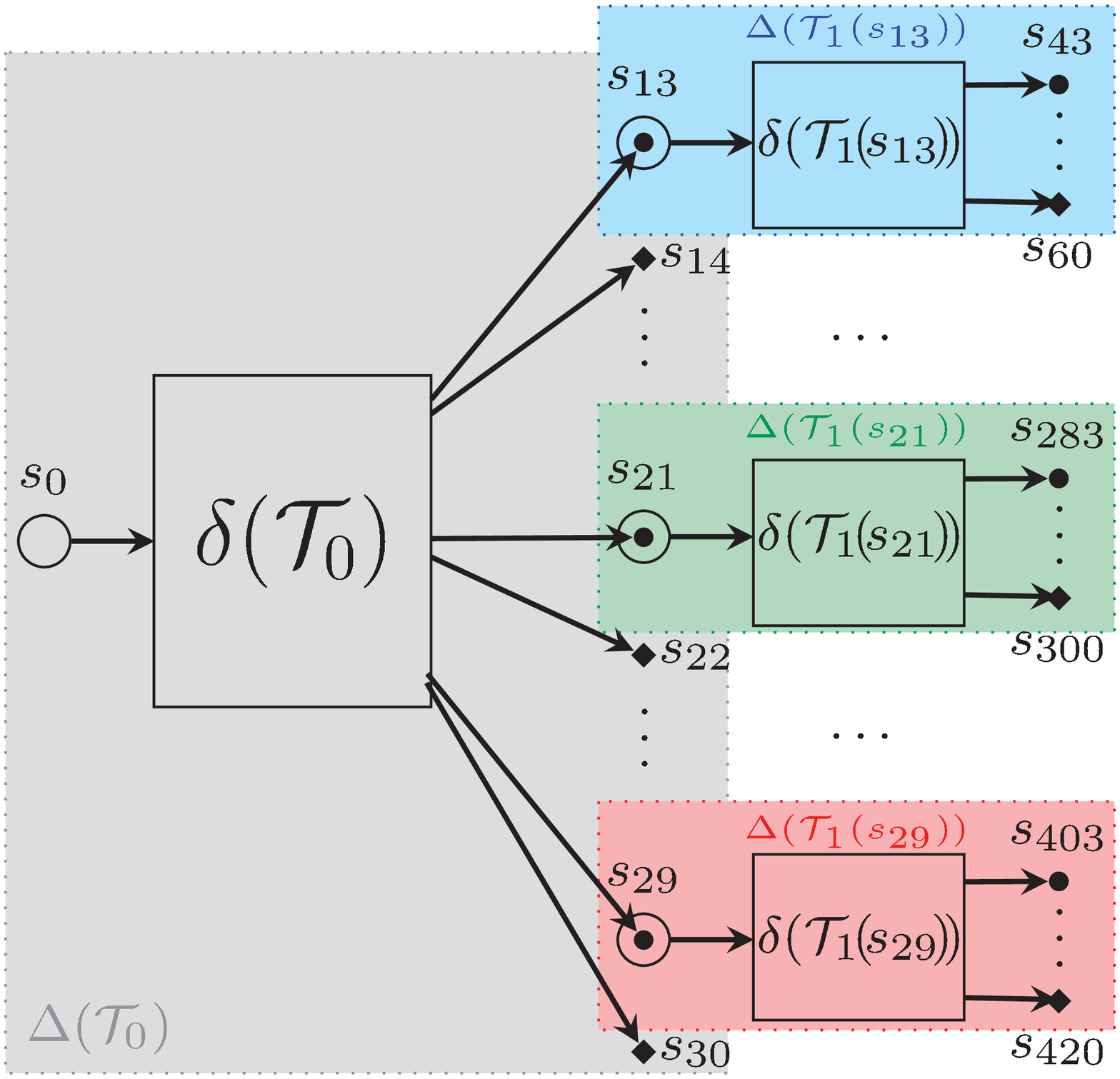}
		\caption{The Event Tree~$\mathcal{T}_1$ associated with Example~\ref{ex:Radicalisation_Dynamic_3variables} that is depicted using event tree objects, each of which is depicted by a dotted rectangle. \label{fig:ObjectOriented_InfiniteEventTree_Radicalisation_3var}}
	\end{minipage}%
	\begin{minipage}{.1\textwidth}
		$\qquad$
	\end{minipage}
	\begin{minipage}{.3\textwidth}
		\vspace{-35pt}
		\centering
		\includegraphics[scale=0.315,angle=-90,origin=c,trim=0 150 -210 0] {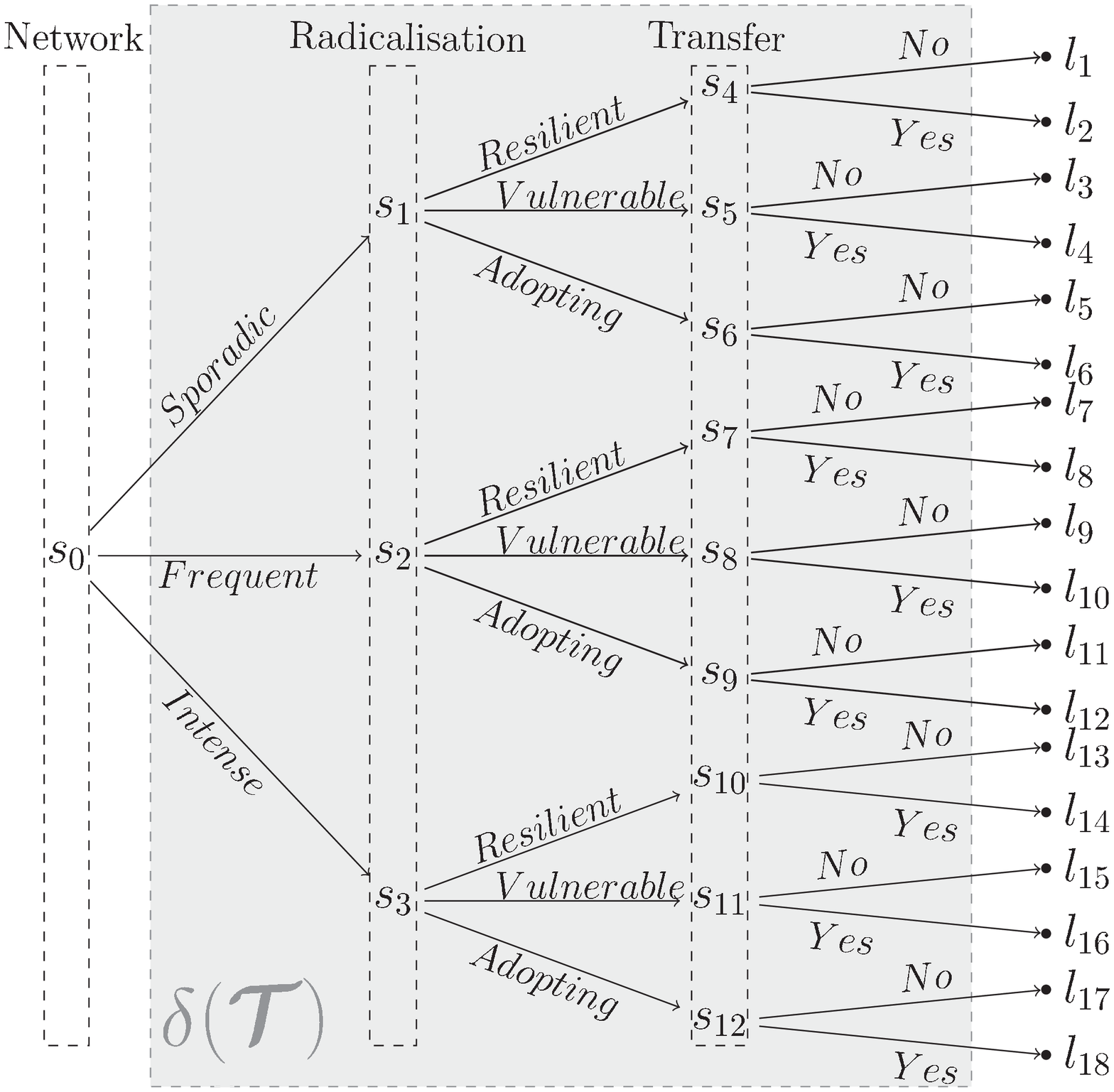}
		\vspace{-35pt}
		\caption{Event Tree~$\mathcal{T}$ \label{fig:FiniteEventTree_Radicalisation}}
	\end{minipage}
\end{figure}

\vspace{0pt}
\begin{figure}[h!]
	\begin{center}
		\includegraphics[scale=1,angle=-90,origin=c,trim=0 23 0 -140]{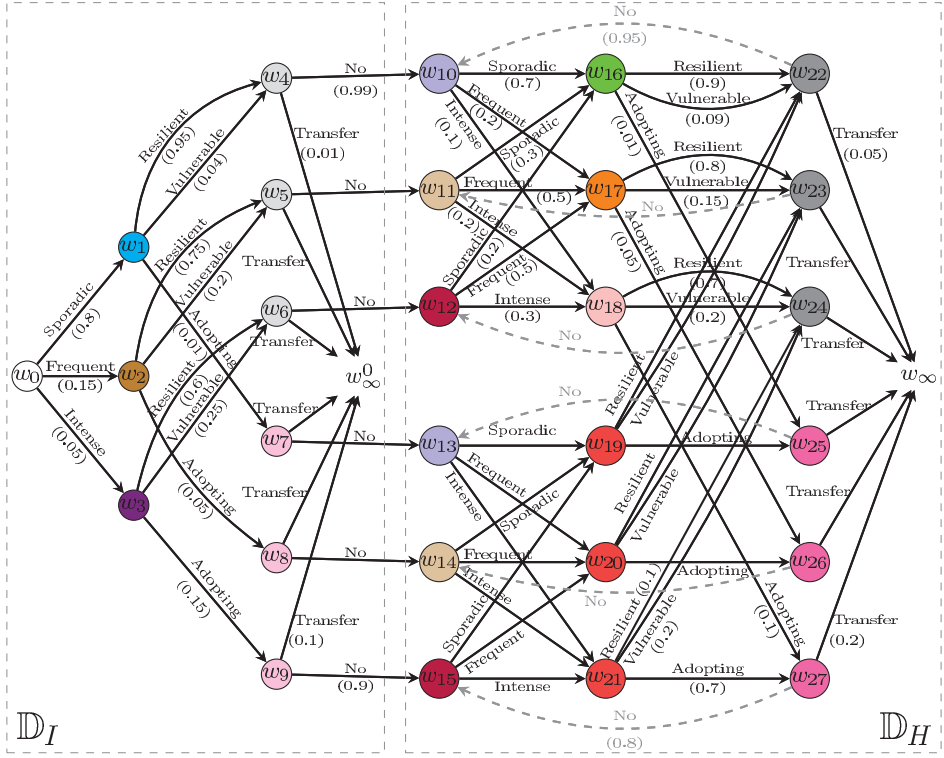}
	\end{center}
	\vspace{-27pt}
	\caption{The 2T-DCEG associated with Example \ref{ex:Radicalisation_Dynamic_3variables}. The stage structure is defined as following: $u_0=\{w_0\}$, $u_1=\{w_1\}$, $u_2=\{w_2\}$, $u_3=\{w_3\}$, $u_4=\{w_4,w_5,w_6\}$, $u_5\!\!=\!\!\{w_7,w_8,w_9\}$, $u_6\!\!=\!\!\{w_{10},w_{13}\}$, $u_7\!\!=\!\!\{w_{11},w_{14}\}$, $u_8\!=\!\{w_{12},w_{15}\}$, $u_9\!=\!\{w_{16}\}$, $u_{10}\!=\!\{w_{17}\}$, $u_{11}=\{w_{18}\}$, $u_{12}=\{w_{19},w_{20},w_{21}\}$,  $u_{13}=\{w_{22},w_{23},w_{24}\}$, $u_{14}=\{w_{25},w_{26},w_{27}\}$. The conditional probabilities associated with a particular stage is shown in parentheses. A dotted rectangle establishes the limit of a particular graph. \label{fig:2T-DCEG_Radicalisation_3var}}
	\vspace{-7pt}
\end{figure}
\end{myExampleCont1}

It is important to recall some useful concepts introduced in~\cite{Collazo.Smith.2018a}. Let $w(t)$ be a position~$w$ of a DCEG or a $T$-position~$w$ of an \text{$N$T-DCEG} at a time~$t$, i.e. corresponding to the non-empty set of situations~$s(t)$ that $w$ merged. In an~\text{$N$T-DCEG} model~$\mathbb{C}$, an edge $(w_a,w_b)$ is said to be a \emph{temporal edge} associated with time-slice~$t$ if and only if we can write it as ${(w_a(t),w_b(t+1))}$. Particularly, a temporal edge associated with time-slices~$t$, $t=N,N\!+\!1,\ldots$, is called a \emph{cyclical temporal edge}. Analogous to a DBN,  the set of all temporal edges and the set of all cyclic temporal edges will be denoted, respectively, by~$E_\dagger$ and $E_\text{\circled{$\dagger$}}$. 

Next remember that every~\text{$N$T-DCEG} graph has two useful subgraphs: the initial graph~$\mathbb{D}_I$ and the cyclic graph~$\mathbb{D}_H$~\cite{Collazo.Smith.2018a}. The subgraph~$\mathbb{D}_I$ represents the initial states of the process over the first $N\!\!-\!\!1$ time-slices and the subgraph $\mathbb{D}_{H}$ depicts the time-homogeneous developments over the subsequent time-slices. Note that $\mathbb{D}_I$~is an acyclic graph. In contrast, $\mathbb{D}_H$~is a cyclic graph because it contains all cyclic temporal edges. For example, assume that a 2T-DBN $\mathbb{D}$ is re-expressed as a 2T-DCEG $\mathbb{C}$ (Corollary~\ref{the:2TDBN_2TSDCEG}). The initial time-slice~$\mathbb{D}(0)$ of $\mathbb{D}$ corresponds to the subgraph $\mathbb{D}_I$ of~$\mathbb{C}$. The  time-homogeneous time-slices $\mathbb{D}(t), t = 1,2,\ldots$, and its associated temporal edges $E_\dagger(t),{t=2,3,\ldots}$, are depicted in the subgraph $\mathbb{D}_H$ of $\mathbb{C}$.

%%%%%%%%%%%%%%%%%%%%%%%%%%%%%%%%%%%%%%%%%%%%%%%%%%%%%%%%%%%%%%
\section{A Stratified Dynamic Chain Event Graph}
\label{sec:SDCEG}

A Stratified Dynamic Chain Event Graph (SDCEG) is a recommended DCEG model when domain experts are more comfortable eliciting the observed process using a set of random variables rather than describing it by sequences of unfolding events. An SDCEG naturally extends the BN framework and so enables domain experts to refine a BN model with context-specific conditional independences and logical constrains.

In our example, we illustrate that exploring these refinements the number of conditional probabilities to be elicited for an SDCEG to fully describe a process is less than that needed for its corresponding DBN. This is desirable for two reasons. First, domain experts can often convey more faithfully the qualitative structures of the observed process than on the quantification of conditional probability distributions. Second, this simplicity associated with an SDCEG implicity embeds sparsity which commonly occurs and is usually enforced in real-world settings. Of course, these advantages need to be weighed against demanding an additional effort from domain experts to construct the staged tree as opposed to a DAG.

Take a non-empty set of random variables $\boldsymbol{\mathcal{Z}}=\boldsymbol{\mathcal{Z}}^{(r)} \cup \boldsymbol{\mathcal{Z}}^{(r,s)}$, where  ${r \in \{1,\ldots,s\!-\!1}\}$,  $\boldsymbol{\mathcal{Z}}^{(r)}=\{\mathcal{Z}_1,\ldots,\mathcal{Z}_r\}$  is a set of $r$~time-invariant variables and ${\boldsymbol{\mathcal{Z}}^{(r,s)}=\{\mathcal{Z}_{r+1},\ldots,\mathcal{Z}_s\}}$ is a set of $r\!-\!s$~variables that take values at each time-slice $t=0,1,\ldots$. Note that $\boldsymbol{\mathcal{Z}}^{(r)}$ may also be an empty set and in this case we convention that $r=0$. Let $I=(I^{(r)},I^{(r,s)})=(i_1,\ldots,i_r,i_{r\!+\!1},\ldots,i_s)$ be a permutation of the set $\{1,\ldots,s\}$, such that 
$I^{(r)}=(i_1,\ldots,i_r)$ and $I^{(r,s)}=(i_{r\!+\!1},\ldots,i_s)$ are, respectively, permutations of the sets~$\{1,\ldots,i_r\}$ and~$\{i_{r\!+\!1},\ldots,i_s\}$. We can use~$I$ to order the set of variables $\boldsymbol{\mathcal{Z}}$ as follows:
\vspace{-7pt}
$$\boldsymbol{\mathcal{Z}}\stackrel{I}{\longmapsto} (\overbrace{\mathcal{Z}_{i_{1}},\ldots,\mathcal{Z}_{i_r}}^{\boldsymbol{{Z}}^{(r)}(I^{(r)})},\overbrace{\mathcal{Z}_{i_{r\!+\!1}},\ldots,\mathcal{Z}_{i_s}}^{\boldsymbol{{Z}}^{(r,s)}(I^{(r,s)})}) \triangleq
\boldsymbol{Z}(I)=(\boldsymbol{{Z}}^{(r)}(I^{(r)}),\boldsymbol{{Z}}^{(r,s)}(I^{(r,s)})),\vspace{-7pt}$$
where $\boldsymbol{Z}(I)$, $\boldsymbol{{Z}}^{(r)}(I^{(r)})$ and~$\boldsymbol{\mathcal{Z}}^{(r,s)}(I^{(r,s)})$ are,respectively, the ordered sequences of the variables in $\boldsymbol{\mathcal{Z}}$, $\boldsymbol{\mathcal{Z}}^{(r,s)}$ and $\boldsymbol{\mathcal{Z}}^{(r,s)}$ spanned by~$I, I^{(r)},I^{(r,s)}$.

Now construct the finite trees~$\mathcal{T}_{-1}(\boldsymbol{{Z}}(I))$ and~$\mathcal{T}(\boldsymbol{{Z}}(I))$ corresponding, respectively, to $\boldsymbol{{Z}}^{(r)}(I^{(r)})$ and~$\boldsymbol{{Z}}^{(r,s)}(I^{(r,s)})$. Recall that an event tree $\mathcal{T}(\boldsymbol{Z})$ based on a sequence of random variables~$\boldsymbol{Z}=(\mathcal{Z}_1,\ldots,\mathcal{Z}_n)$ is one whose set of paths $\Lambda(\mathcal{T(\boldsymbol{Z})})$ can be expressed as a product space associated with the sequence of random variables $\mathcal{Z}_1,\ldots,\mathcal{Z}_n$. For more details on this construction, see \cite{Cowell.Smith.2013,Barclay.etal.2013a, Collazo.etal.2018} and the proof of Theorem~\ref{the:DBN_SDCEG}. 

\begin{myDef}
\label{def:SDCEG}
A DCEG is called a $\boldsymbol{\mathcal{Z}}-$\textbf{Stratified Dynamic Chain Event Graph} ($\boldsymbol{\mathcal{Z}}-$SDCEG) when its staged tree satisfies the following conditions:
\begin{enumerate}[nosep]
\item For some valid permutation~$I$, its supporting event tree can be expressed as a TOG($\mathcal{T}_{-1}(\boldsymbol{Z}(I)),\mathcal{T}(\boldsymbol{Z}(I))$).
\item Each stage only merges situations associated with the same random variable, although they may be at different time-slices.
\end{enumerate}
A $\boldsymbol{\mathcal{Z}}-$\textbf{SDCEG model} associated with a $\boldsymbol{\mathcal{Z}}-$SDCEG~$\mathbb{C}$ is a graphical model whose sample space is represented by the supporting event tree of $\mathbb{C}$ and whose probability measure respects the set of conditional independence statements depicted by~$\mathbb{C}$. A staged tree that satisfies the two conditions above are said to be a $\boldsymbol{\mathcal{Z}}-$\textbf{Stratified Staged Tree}. If the staged tree of an $N$T-DCEG is $\boldsymbol{\mathcal{Z}}$-stratified, we obtain an $\boldsymbol{N}$ \textbf{Time-Slice} $\boldsymbol{\mathcal{Z}}-$\textbf{SDCEG} ($N$T$\boldsymbol{\mathcal{Z}}$-SDCEG). 
\end{myDef}

This framework can be particularly useful for causal analyses (see Section~\ref{sec:local_global_independence}) when the total or partial variable orders of the vectors   $\boldsymbol{\mathcal{Z}}^{(r)}$ and $\boldsymbol{\mathcal{Z}}^{(r,s)}$ imply different causal hypotheses. The SDCEG class is also an important model family because Theorem~\ref{the:DBN_SDCEG} tells us that every DBN can be rewritten as an SDCEG. In particular, according to Corollary~\ref{the:2TDBN_2TSDCEG} every 2T-DBN can be translated into a 2T-SDCEG. These results enable us to embellish a DBN with context-specific statements using the broader class of DCEG models. This can be helpful for model search since a DCEG model space is considerably larger than its corresponding DBN model space. For example, the 2T-DBN model selection can be used as a starting point for a 2T-DCEG model search. For an application of such model search strategy using CEGs and BNs, see \cite{Barclay.etal.2013a}. Alternatively, a DCEG model can provide a framework for constructing random variables (see Section \ref{sec:Constructiong_Random_Variable}) which in turn enable us to express the context-specific statements using a DBN model. 
  
\begin{myTheorem}
\label{the:DBN_SDCEG}
All conditional independence statements entailed by the ordered Markov property  in a DBN can be depicted by an SDCEG.
\end{myTheorem}
\begin{proof}
See \ref{app:DBN_SCDEG}.
\end{proof}

\begin{myCorollary}
\label{the:2TDBN_2TSDCEG}
Every conditional independence statements defined by the ordered Markov property  in a $\text{$N\!$T-DBN}$ can be expressed in a \text{$N\!$T-SDCEG}.
\end{myCorollary}
\begin{proof}
See \ref{app:2TDBN_2TSCDEG}.
\end{proof}

\begin{myExampleCont1}
\emph{
The 2T-DBN and the 2T$\boldsymbol{\mathcal{Z}}$-SDCEG corresponding to the radicalisation dynamic described in Example \ref{ex:Radicalisation_Dynamic_3variables} are given, respectively, in Figures \ref{fig:DBN_Radicalisation} and \ref{fig:2T-DCEG_Radicalisation_3var}. In this case, $\boldsymbol{\mathcal{Z}}=\{N,R,T\}$ and $\boldsymbol{Z(I)}=(N,R,T)$. Figure~\ref{fig:FiniteEventTree_Radicalisation} shows how to construct the event tree~$\mathcal{T}(\boldsymbol{{Z}}(I))$.  It can be easily verified that every conditional independence statement depicted in the \text{2T-DBN} 
is also showed in the 2T$\boldsymbol{\mathcal{Z}}$-SDCEG.}

\emph{However only the symmetric conditional independences exhibited in the 2T$\boldsymbol{\mathcal{Z}}$-SDCEG (Figure~\ref{fig:2T-DCEG_Radicalisation_3var}) can be graphically read from the 2T-DBN (Figure~\ref{fig:DBN_Radicalisation}). To illustrate this, take the variable Radicalisation~R. The context-specific conditional independences associated with this variable are directly depicted in the 2T$\boldsymbol{\mathcal{Z}}$-SDCEG. For instance, we can see from the graph that the probability of deradicalisation in time $t+1$ given that a prisoner has already adopted radicalisation in time $t,t \geq 1$, (positions $w_{19}$, $w_{20}$, $w_{21}$) is independent of his social contacts in the prison since positions $w_{19}$, $w_{20}$ and $w_{21}$ are coloured the same (red). This is not so for the 2T-DBN. Also observe that we can read from the 2T$\boldsymbol{\mathcal{Z}}$-SDCEG that the variable~$T$ is associated with a terminating event, although this kind of logical constraint cannot be directly read from the corresponding 2T-DBN. Due to these qualitative structures uncovered by the 2T$\boldsymbol{\mathcal{Z}}$-SDCEG, domain experts only needs to elicit 14 conditional probability distributions since there are 14 stages. On the other hand, they will have to elicit 21 conditional probability distributions in the corresponding BN framework if they ignore the hidden context-specific conditional independences}.
\end{myExampleCont1}

%%%%%%%%%%%%%%%%%%%%%%%%%%%%%%%%%%%%%%%%%%%%%%%%%%%%%%%%%%%%%%
\section{Composite models}
\label{sec:Composite_Model}

An $N$T-DCEG model provides experts with a flexible and useful framework for composite model construction. Building on the algorithm for constructing a standard $N$T-DCEG model~\cite{Collazo.Smith.2018a}, Algorithm~\ref{alg:NTDCEG_algorithm} introduces a formal methodology for domain experts, decision makers and modellers to work in parallel and find a common ground that is coherent and consistency with their collective assumptions and degrees of belief. For this purpose, it is assumed that a set of time-invariant characteristics distinguishes different types of units observed in the system. Domain experts should use these time-invariant features to elicit the finite event tree~$\mathcal{T}_{\!-1}$ and then split the modelling responsibility between different panels. Note that each leaf of $\mathcal{T}_{\!-1}$ defines a particular type of units. Of course, if the set of units is homogeneous, then this algorithm cannot be used for composite model construction. However, in these cases a different type of domain information can be explored in order to adapt this methodology for composite model construction.

\begin{algorithm}[htbp]
  \DontPrintSemicolon
  
  %\KwIn{Finite Event Trees~$\mathcal{T}_{-1}$ and~$\mathcal{T}$}
  
  \KwOut{Composite $N$T-DCEG model}
  
  Construct a finite event tree~$\mathcal{T}_{\!-1}$.\;
  
  Split the modelling work between expert panels using~$\mathcal{T}_{\!-1}$.\; 
  
  Each panel~$\mathscr{P}^i$ independently elicit an even tree~$\mathcal{T}^i$ corresponding to its subprocess.\;
  
  Discussion between panels to find a unique even tree~$\mathcal{T}$.\;
  
  Each panel~$\mathscr{P}^i$ defines an $N$T-DCEG graph~$\mathbb{C}^i$ associated with its subprocess that is supported by~$\mathcal{T}$  using the algorithm in~\cite{Collazo.Smith.2018a}.\;
  
  Discussion between panels to find a colour agreement between their different $N$T-DCEG~$\mathbb{C}^i$.
  
  Obtain the final $N$T-DCEG model~$\mathbb{C}$ using the algorithm in~\cite{Collazo.Smith.2018a}.\; 
    
  \Return{$\mathbb{C}$ and $P(\mathbb{C})$}\;
  
\caption{Composite $N$T-DCEG Algorithm \label{alg:NTDCEG_algorithm}}
\end{algorithm}              

Next each panel of experts~$\mathscr{P}^i$ needs to elicit a finite event tree~$\mathcal{T}^i$ that enables them to describe the dynamic process associated with the type of units under their responsibility. Since a fundamental assumption is that all units follows the same underlying dynamic, the different panels should gather and discuss the qualitative description of their subprocesses in order to obtain a single and common finite event tree~$\mathcal{T}$. Working first in small groups of experts focused on a more homogeneous set of units minimises the risk of paralysis due to many conflictive world views and stimulates a very detailed analysis of the process. After it is easier to obtain an aligned description of the process across panels since the set of event trees~$\mathcal{T}^i$ provides a rational base and a common language to discuss the different world views and understandings. In fact, we have observed that the panels tend to be more open-minded for other descriptions and try to incorporate events that they may previously neglect.

Using the algorithm introduced in~\cite{Collazo.Smith.2018a}, each panel~$\mathscr{P}^i$ can now construct an $N$T-DCEG~$\mathbb{C}^i$ supported by a TOG($\mathcal{T}_{\!-1}^i,\mathcal{T}$), where $\mathcal{T}_{\!-1}^i$ identifies the different types of units under the modelling responsibility of~$\mathscr{P}^i$. At this step each panel will elicit the stage structure and the conditional probabilities that drive their units over time in the observed system. After the panels need to compare the stage structure of their $N$T-DCEG models~$\mathbb{C}^i$ in order to verify if stages from different~$\mathbb{C}_i$ can be merged and so find a colour agreement. Here the concept of $T$-position is very important since it enforces a common graphical structure between the set of $N$T-DCEGs~$\mathbb{C}^i$. Note that the refinement of the conditional probability distributions~$P(\mathbb{C}^i)$ elicited in step~$5$ is postponed for the last step when the experts review them using the composite $N$T-DCEG graph~$\mathbb{C}$ and define~$P(\mathbb{C})$. The example below illustrates the composite model construction for the radicalisation dynamic.         

%\vspace{10pt}
\begin{myExample}
	\label{ex:NTDCEG_Radicalisation_Covariate}
	\emph{
		Assume that Example~\ref{ex:Radicalisation_Dynamic_3variables} refers to a British prison. Now suppose that the prison manager contracted a group of experts to gain some insight about the radicalisation dynamic over time. He told experts that he believes that factors such as prisoners' previous convection and nationality appears to drive this process. The experts decide to use the $N$T-DCEG framework as prescribed in Algorithm~\ref{alg:NTDCEG_algorithm}. Based on the manager's description they constructed the event tree $\mathcal{T}_{\!-1}$ depicted in Figure~\ref{fig:2T-DCEG_Radicalisation_Covariate} that divides the inmates in four heterogeneous groups as a function of two time-invariant features:
		\begin{itemize}[nosep]
			\item Conviction, a categorical variable signalising the existence of a prior convection ($y$ - Yes) or not ($n$ - No); and
			\item Nationality, a categorical variable distinguishing between a British inmate ($b$) and a foreigner ($f$).   
		\end{itemize}
	}

	\vspace{-30pt}
	\begin{figure}[h!]
		\begin{center}
			\includegraphics[scale=0.75,angle=-90,origin=c,trim=0 0 0 -150]{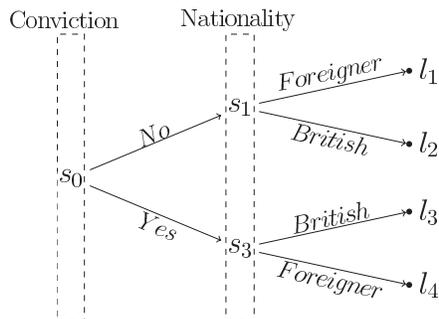}
		\end{center}
		\vspace{-39pt}
		\caption{The time-invariant event tree~$\mathcal{T}_{\!-1}$ associated with Example~\ref{ex:NTDCEG_Radicalisation_Covariate}. \label{fig:covariate_event_tree}}
		\vspace{-7pt}
	\end{figure}
	
	\emph{
		The experts then split the modelling task into two panels according to the variable~$Conviction$: 
		\begin{itemize}[nosep]
			\item Panel 1 - these experts are responsible for modelling the radicalisation process associated with prior convicted inmates.
			\item Panel - the experts are in charge of modelling the radicalisation dynamic observed in groups of prisoners that have no prior conviction.
		\end{itemize}
		Each panel started their work and independently concluded that the event tree showed in Figure~\ref{fig:FiniteEventTree_Radicalisation} is appropriate to represent the process at each time-slice. Next each panel separately realized that the prison dynamic is independent of prisoners' nationality since this process can be expected to be driven by cultural and social factors. The nationality is then not an appropriate explanatory variable. Panel~$1$  assumed that the radicalisation process of a prior convicted prisoner is represented by the 2T-DCEG depicted in Figure~\ref{fig:2T-DCEG_Radicalisation_3var} which embeds all the hypotheses described in Example~\ref{ex:Radicalisation_Dynamic_3variables}.
	 }
		
	\emph{
		Panel~$2$ found that the $2$T-DCEG shown in Figure~\ref{fig:2T_DCEG_Radicalisation_Female} represents the radicalisation dynamic of a prisoner having no criminal history. In line with domain experts' information, this structure is simpler than that of a prisoner with previous convictions since additional context-specific conditional independences to those hypothesised in Example~\ref{ex:Radicalisation_Dynamic_3variables} apply to this case.
	}
	
	\vspace{13pt}
	\begin{figure}[!ht]
		\vspace{15pt}
		\centering
		\begin{subfigure}{.56\textwidth}
			\vspace{-5pt}
			\begin{center}
				$\quad$
				\includegraphics[scale=0.83,angle=-90,origin=c,trim=80 160 0 0]
				{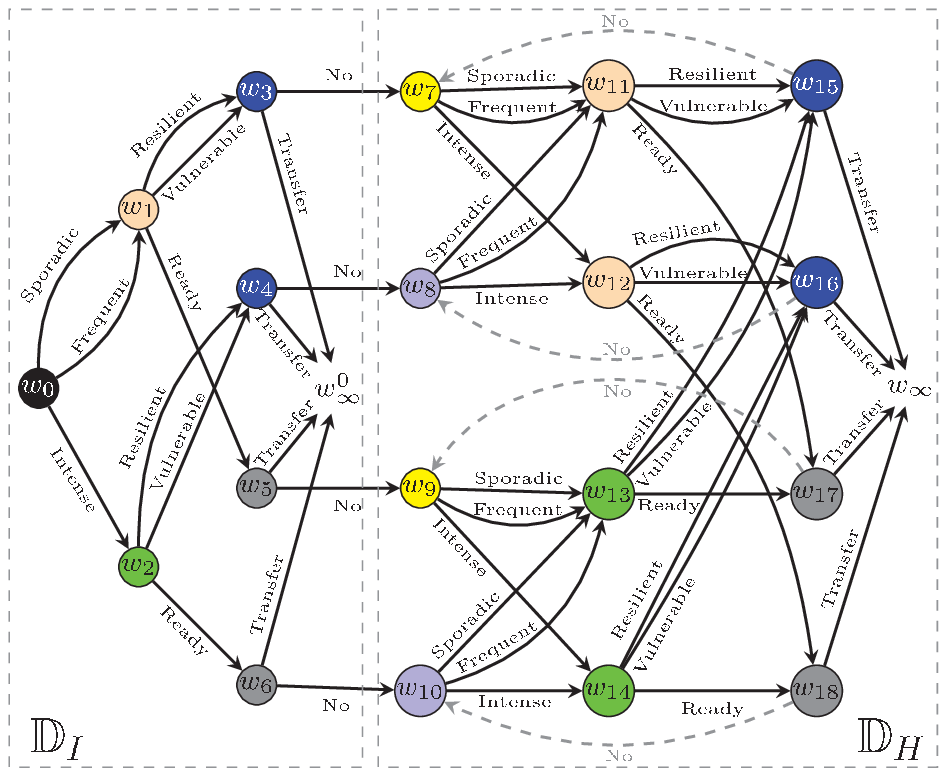}
			\end{center}
			\vspace{45pt}
			\caption{
				2T-DCEG. The stage structure is defined as following: ${u_a=\{w_0\}}$, $u_b=\{w_1\}$, $u_c=\{w_2,w_{13},w_{14}\}$, $u_d=\{w_3,w_4,w_{15},w_{16}\}$, $u_e=\{w_5,w_6,w_{17},w_{18}\}$, $u_f\!\!=\!\!\{w_7,w_9\}$, ${u_g\!\!=\!\!\{w_{8},w_{10}\}}$.
				%$\!\!\!\!\!\!\!\!\!\!\!\!\!\!\!\!\!\!\!\!\!$
				 \label{fig:2T_DCEG_Radicalisation_Female}}
		\end{subfigure}%
		\begin{minipage}{.02\textwidth}
			$ $
		\end{minipage}
		\begin{subfigure}{.365\textwidth}
				\begin{center}
					$\!\!\!\!\!\!\!\!\!\!\!\!\!\!\!\!\!\!\!\!\!\!\!\!\!\!\!\!\!\!\!\!\!\!\!\!\!\!\!$
					\includegraphics[scale=0.75,angle=0,origin=c,trim=0 150 0 -45]
					{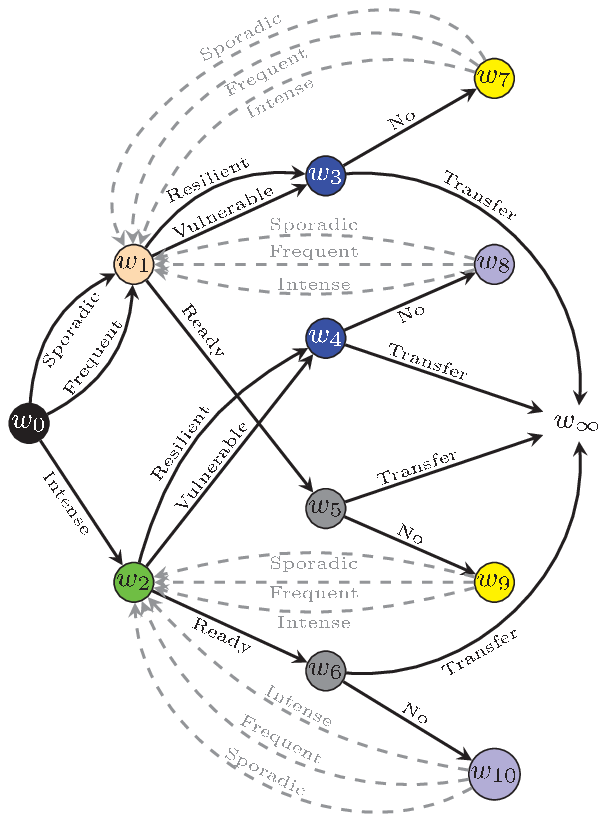}
				\end{center}
				\vspace{70pt}
				\caption{DCEG. 
					The stage structure is given by: ${u_a=\{w_0\}}$, $u_b=\{w_1\}$, ${u_c=\{w_2\}}$, ${u_d=\{w_3,w_4\}}$, ${u_e=\{w_5,w_6\}}$, ${u_f\!\!=\!\!\{w_7,w_9\}}$, ${u_g\!\!=\!\!\{w_{8},w_{10}\}}$. %$\!\!\!\!\!\!\!\!\!\!\!\!\!\!\!\!\!\!\!\!\!\!\!\!\!\!\!\!\!\!\!\!\!\!\!\!$ 
					\label{fig:DCEG_Radicalisation_Female}}
				\vspace{10pt}
			\end{subfigure} 
		\vspace{-10pt}
		\caption[The 2T-DCEG and DCEG associated with the radicalisation process]{The 2T-DCEG and a DCEG corresponding to the radicalison process of prisoners without prior criminal convictions in Example~\ref{ex:NTDCEG_Radicalisation_Covariate}.
		\label{fig:DCEGs_Radicalisation_Female}}
		\vspace{-7pt}
	\end{figure}
				
	\emph{
		Figure~\ref{fig:DCEG_Radicalisation_Female} depicts the DCEG equivalent to the 2T-DCEG presented in Figure~\ref{fig:2T_DCEG_Radicalisation_Female} in a sense that both models embed the same set of conditional independence statements. Note that in the DCEG graph some situations in the initial time-slice are in positions that also aggregate situations that unfold in the subsequent time-slices. This gives a very simple graph with only four levels. However, this simplification has some drawbacks in terms of the} readability \emph{of the conditional independences represented by the DCEG model, which do not happen in the ${2\text{T-DCEG}}$ framework.
		}
		
	\emph{
		For example, position~$w_2$ of the DCEG corresponds to an inmate who has intense social contacts with other radicalised prisoners during the initial time-slice. It also merges situations in time-slice~$t$, ${t=1,2,\ldots}$, that correspond to a radicalised inmate in the previous time-slice $t\!-\!1$ who remains in prison at the current time-slice~$t$. This last statement cannot be read immediately from the DCEG topology. In contrast, we can read this statement directly from the ${2\text{T-DCEG}}$ topology if we look at positions~$w_{13}$ and~$w_{14}$ since its position~$w_2$ only represents the first type of inmates. Note that these three positions are at the same stage~$u_c$ of the ${2\text{T-DCEG}}$. Adopting the concept of a $1$-position has thus enabled us to disentangle the initialisation process showed in subgraph~$\mathbb{D}_I$ from the transition process presented in subgraph~$\mathbb{D}_H$. The subgraph~$\mathbb{D}_I$ associated with the initial time-slice can then act as a legend to analyse the subsequent time-homogeneous time-slices depicted in the subgraph~$\mathbb{D}_H$.
	}

	\emph{
		Observe now that it is easier to compare the radicalisation processes of an inmate with a prior or a non-prior conviction using the 2T-DCEG. Assume that experts of panels~$1$ and~$2$ agreed to merge only the following three pairs of stages: $u_c$ and $u_{9}$; $u_e$ and $u_{14}$; $u_g$ and $u_6$. Also suppose that they accept that the radicalisation risk  of a non-prior convicted prisoner is lower than a prior convicted prisoner with similar behaviour pattern in the prison. It is straightforward to verify that merging the 2T-DCEGs depicted in Figures~\ref{fig:2T-DCEG_Radicalisation_3var} and~\ref{fig:2T_DCEG_Radicalisation_Female} under these assumptions we obtain the 2T-DCEG illustrated in Figure~\ref{fig:2T-DCEG_Radicalisation_Covariate}. 
	}
	
	%\vspace{5pt}
	\begin{figure}[ht!]
		\begin{center}
			\includegraphics[scale=0.619,origin=c,trim=100 0 0 -111]
			{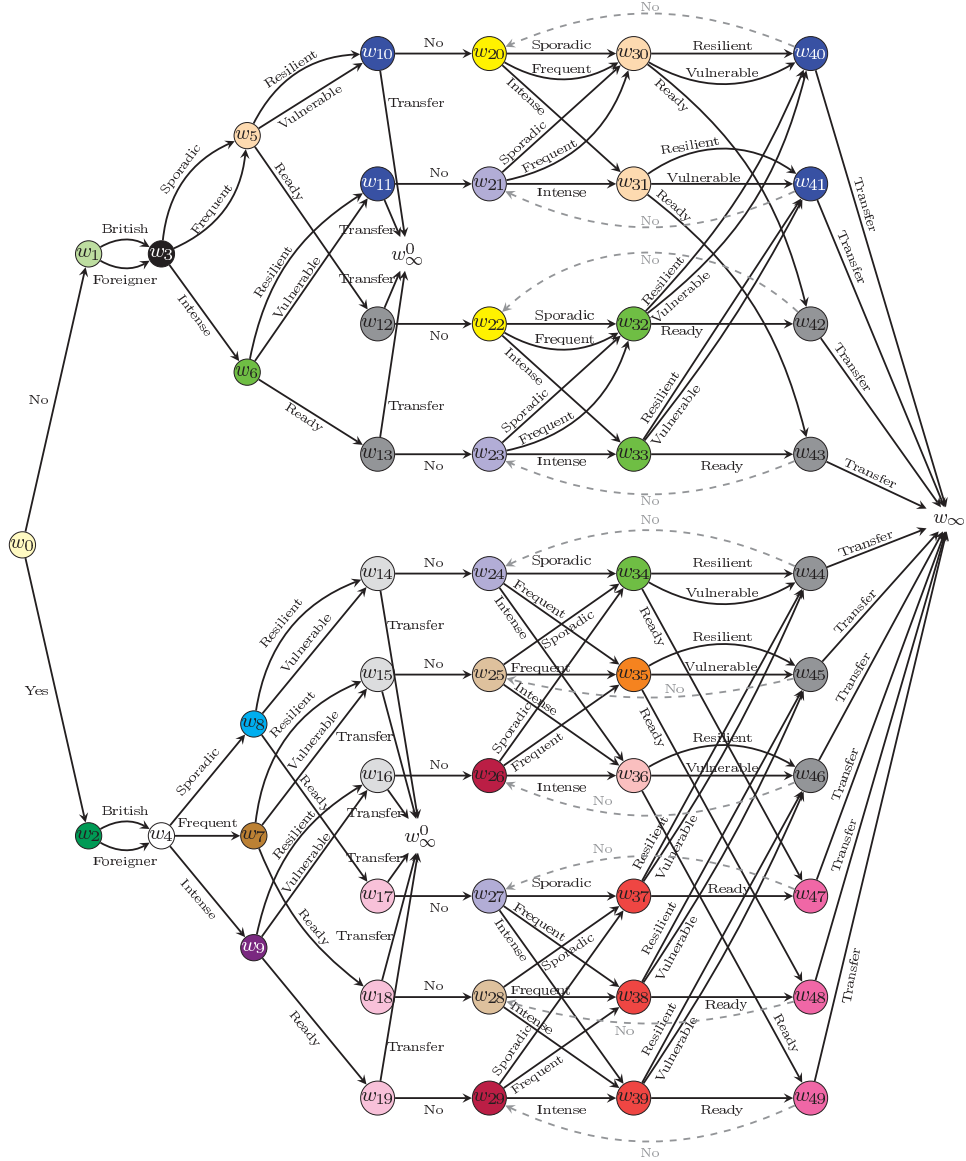}
		\end{center}
		\vspace{-17pt}
		\caption[A 2T-DCEG with time-invariant variables]{
			The 2T-DCEG associated with Example~\ref{ex:NTDCEG_Radicalisation_Covariate}. 
			The stage structure is defined as following: ${u_0=\{w_0\}}$, ${u_1=\{w_1\}}$, ${u_2=\{w_2\}}$, $u_3=\{w_3\}$, $u_4=\{w_4\}$, $u_5=\{w_5,w_{30},w_{31}\}$, ${u_6=\{w_6,w_{32},w_{33}\}}$, ${u_7=\{w_{7}}\}$, $u_8=\{w_{8}\}$, $u_9=\{w_{9}\}$, $u_{10}=\{w_{10},w_{11},w_{40},w_{41}\}$, ${u_{11}=\{w_{12},w_{13},w_{42},\ldots,w_{46}\}}$, ${u_{12}=\{w_{14},w_{15},w_{16}\}}$, $u_{13}=\{w_{17},w_{18},w_{19}\}$, ${u_{14}=\{w_{20},w_{22}\}}$, ${u_{15}=\{w_{21},w_{23},w_{24},w_{27}\}}$,  ${u_{16}=\{w_{25},w_{28}\}}$, ${u_{17}=\{w_{26},w_{29}\}}$, ${u_{18}=\{w_{30},w_{31}\}}$, ${u_{19}=\{w_{32},w_{34}\}}$, ${u_{20}=\{w_{35}\}}$, ${u_{21}=\{w_{36}\}}$, ${u_{22}=\{w_{37},w_{38},w_{39}\}}$, ${u_{23}=\{w_{47},w_{48},w_{49}\}}$.
			Hotter colours implies higher risk of radicalisation.
			\label{fig:2T-DCEG_Radicalisation_Covariate}}
		\vspace{-7pt}
	\end{figure}
	
\end{myExample}

%%%%%%%%%%%%%%%%%%%%%%%%%%%%%%%%%%%%%%%%%%%%%%%%%%%%%%%%%%%%%%
%\section{Learning an $N$T-DCEG}
%\label{sec:learning}

%%%%%%%%%%%%%%%%%%%%%%%%%%%%%%%%%%%%%%%%%%%%%%%%%%%%%%%%%%%%%%
\section{Constructing random variables}
\label{sec:Constructiong_Random_Variable}

Sometimes an SDCEG~$\mathbb{C}$ can also be described by a context-specific DBN. In this case a useful class of random variables is one taking its levels as positions that are equally distant from the root position in~$\mathbb{C}$. However especially when its tree is asymmetric such random variables are not the only or even the most important class of random variables that can be constructed from an ${\text{DCEG}}$.

In this section we will present two constructions of random variables intrinsically associated with an ${N\text{T-DCEG}}$~$\mathbb{C}\!\!=\!\!(V,E)$. We will then show how particularly useful conditional independences can be defined between them. This will first require us to extend the concepts of cut and fine cut from a CEG~(\cite{Smith.Anderson.2008}) so that they can be interpreted analogously in an ${N\text{T-DCEG}}$. 

For an $N$T-DCEG $\mathbb{C}$, define the following finite set of positions associated with the temporal edges from time~$t, t=-1,\ldots,N-2$, to time~${t+1}$:
\vspace{-9pt}
$${\mathcal{W}_{Tail}^t = \{w(t);(w(t),w_a(t+1)) \in E_{\dagger},  \text{ for some } w_a(t+1)\}} \text{ and}\vspace{-9pt}$$
$${\mathcal{W}_{Head}^{t+1} = \{w(t+1);(w_a(t),w(t+1)) \!\in\! E_{\dagger},  \text{ for some } w_a(t)\}}.\vspace{-5pt}$$
For $t=N-1,N,\ldots$, fix ${\mathcal{W}_{Tail}^t=\mathcal{W}_{Tail}}$, ${\mathcal{W}_{Head}^t=\mathcal{W}_{Head}}$ and $w_\infty^t=w_\infty$.
Take the graph $\mathbb{C}^- \!\!=\! (V^-,E^-)$, where $V^- \!\!=\!\! V$ and $E^- \!\!=\!\! E \!-\! E_\text{\circled{$\dagger$}}$, and let ${\mathcal{W}_{Tail}^{t(*)}=\mathcal{W}_{Tail}^t \cup \{w_\infty^t\}}$. Now let $\Lambda_t^{cut}$, ${t=0,1,\ldots}$, denote the set of all \text{$w_0$-to-$\mathcal{W}_{Head}^t$-to-$\mathcal{W}_{Tail}^{t(*)}$} paths in $\mathbb{C}^-$. For $t=-1$, when $\mathcal{T}_{-1}\neq\emptyset$, and for $t=0$, when $\mathcal{T}_{-1}=\emptyset$, let $\Lambda_t^{cut}$ be the set of all ${w_0\text{-to-}\mathcal{W}_{Tail}^{t(*)}}$ paths in $\mathbb{C}^-$. 
	
	\begin{myDef}
		\label{def:Cut}
		In an $N$T-DCEG~$\mathbb{C}$, a \textbf{cut} $\mathcal{U}_t^{cut}$, $t=0,1,\ldots$, is a set of stages such that all paths in $\Lambda_t^{cut}$ pass through exactly one position $w(t) \in u$, for some $ u \in \mathcal{U}_t^{cut}$. Let $\mathcal{U}^{cut}$ denote any of the identical cuts $\mathcal{U}_t^{cut}$, $t=N-1,N,\ldots$. 
	\end{myDef}
	\vspace{0pt}
	
	\begin{myDef}
		\label{def:Fine_cut}
		In an $N$T-DCEG~$\mathbb{C}$, a \textbf{fine cut} $\mathcal{W}_t^{cut}$, $t=-1,0,\ldots$, is a set of positions $w(t)$ such that all paths in $\Lambda_t^{cut}$ pass through exactly one position ${w(t) \in \mathcal{W}_t^{cut}}$. Let $\mathcal{W}^{cut}$ denote any of the identical cuts $\mathcal{W}_t^{cut}$, $t=N-1,N,\ldots$.
	\end{myDef} 
	
Define a function a $h$ such as $h(x)=0$, if $x\leq N-1$, and $h(x)=x-N+1$, otherwise. Let $\Lambda(\mathcal{U}_t^{cut})=\cup_{u \in \mathcal{U}_t^{cut}}\Lambda(u,t)$, where $\Lambda(u,t)$ denotes the set of all walks ${\lambda=(w_{0},\ldots,w{(t)}) \subset \mathbb{C}}$, such that $w(t) \in u$  and each walk~$\lambda$ passes through cycle temporal edges exactly $h(t)$ times. 
Remember that in a graph~$\mathbb{G}=(V_\mathbb{G},E_\mathbb{G})$ a walk corresponds to a non-empty sequence of vertices~$(v_{i_1},v_{i_2},\ldots,v_{i_L})$, such that the vertices are not necessarily distinct and every edge $(v_{i_k},v_{i_{k+1}})$, $k=0,\ldots,L\!-\!1$, is in~$E_\mathbb{G}$~\cite{Cowell.etal.2007,Diestel.2006}. Also let $\mathcal{E}(u)$ and $\mathcal{E}(\mathcal{U}_t^{cut})$ denote the set of events that can happen immediately after a unit arriving, respectively, at a particular stage~$u$ and at any stage in a cut~$\mathcal{U}_t^{cut}$. We can now introduce three useful random variables that can be constructed from the cut $\mathcal{U}^{cut}$ taking values over time-slices~$t, t=-\!1,0,\ldots$. These are defined as follows: 
	\begin{enumerate}[nosep]
		\item $X(\mathcal{U}_t^{cut})$ is defined to be a random variable whose state space is the set $\mathbb{X}(\mathcal{U}_t^{cut})=\{1,2,\ldots,|\mathcal{E}(\mathcal{U}_t^{cut})|\}$, such that there exists a bijection ${\zeta_{X(\mathcal{U}_t^{cut})}:\mathbb{X}(\mathcal{U}_t^{cut}) \to \mathcal{E}(\mathcal{U}_t^{cut})}$. Its probability mass function $\pi_{X}(x)$ is given by
		\vspace{-5pt} 
		\begin{equation}
		\label{eq:X_cut_mass_function}
		\pi_{X}(x) \propto \!\!\!
		\sum_{\lambda \in \Lambda_x(\mathcal{U}_t^{cut})} \!\!\! \!\!
		\pi (w'(l(\lambda))=\zeta_{X}(x)|l(\lambda))
		\prod_{\substack{w \in \lambda \\ w \neq l(\lambda)}} \!\!\! \pi(w'(w)|w),
		\, x \in \mathbb{X}(\mathcal{U}_t^{cut}),
		\vspace{0pt}
		\end{equation}
		where $w'(w)$ is the successor of $w$ in $\lambda$, 
		$l(\lambda)$ is the last position of a directed walk~$\lambda$, and $\Lambda_x(\mathcal{U}_t^{cut})$ is the set of all walks
		$\lambda \in \Lambda(\mathcal{U}_t^{cut})$, such that the event $\zeta_{X(\mathcal{U}_t^{cut})}(x)$ can unfold from $\lambda$ in $\mathbb{C}$.
		
		\item ${Q}(\mathcal{U}_t^{cut})$ is defined to be a random variable whose state space is given by the set ${\mathbb{Q}(\mathcal{U}_t^{cut})=\{1,2,\ldots,|\mathcal{U}_t^{cut}|\}}$, such that there is a bijection $\zeta_{Q(\mathcal{U}_t^{cut})}:\mathbb{Q}(\mathcal{U}_t^{cut}) \to \mathcal{U}_t^{cut}$. The probability mass function  $\pi_{Q}(q)$ is proportional to the sum of all the monomials in primitives associated with $\lambda \!\in\! \Lambda(u,t)$, where $u=\zeta_{Q(\mathcal{U}_t^{cut})}(q)$. So explicitly we have that
		\vspace{-5pt}
		\begin{equation}
		\label{eq:Q_cut_mass_function}
		\pi_{Q}(q) \propto 
		\sum_{\lambda \in \Lambda(\zeta_Q(q),t)}
		\prod_{\substack{w \in \lambda \\ w \neq l(\lambda)}} 
		\pi(w'(w)|w), \quad q \in \mathbb{Q}(\mathcal{U}_t^{cut}).
		\vspace{-5pt}
		\end{equation}
		
		\item $Z(\mathcal{U}_t^{cut})$ is the upstream random variable of  $\mathcal{U}_t^{cut}$ in $\mathbb{C}$ whose state space is 
		defined by the set ${\mathbb{Z}(\mathcal{U}_t^{cut})=\{1,2,\ldots,|\Lambda(\mathcal{U}_t^{cut})|\}}$, such that there is a bijection $\zeta_{Z(\mathcal{U}_t^{cut})}:\mathbb{Z}(\mathcal{U}_t^{cut}) \to \Lambda(\mathcal{U}_t^{cut})$. Its probability mass function $\pi_Z(z)$ is given by
		\vspace{-5pt}
		\begin{equation}
		\label{eq:Z_cut_mass_function}
		\pi_{Z}(z) \propto
		\prod_{\substack{w \in \lambda \\ w \neq l(\lambda)}}
		\pi(w'(w)|w),
		\quad z \in \mathbb{Z}(\mathcal{U}_t^{cut}),
		\vspace{-5pt}
		\end{equation}
		where $\lambda=\zeta_{Z(\mathcal{U}_t^{cut})}(z)$.
	\end{enumerate}
	
	Note that `$=$' can replace `$\propto$' in the three equations above if the $N\text{T-DCEG}$ does not have a sink position~$w_\infty$. Let $\mathbb{X}(u)=\{1,2,\ldots,|\mathcal{E}(u)|\}$ be the state space of the usual random variable~$X(u)$ associated with a stage $u$. From the constructions above we can now immediately recover each random variable $X(u)$, $u \in \mathcal{U}_t^{cut}$, as follows 
	\begin{equation}
	\pi(X(\mathcal{U}_t^{cut})=x|Q(\mathcal{U}_t^{cut})=q) = \left\{
	\begin{array}{rl}
	\pi(X(u_q)=x|u_q) & \text{if } \zeta_{X(\mathcal{U}_t^{cut})}(x) \in \mathcal{E}(u_q) ,\\
	0 &\text{if } \zeta_{X(\mathcal{U}_t^{cut})}(x) \notin \mathcal{E}(u_q),
	\end{array} \right.
	\end{equation}
	where $u_q=\zeta_{Q(\mathcal{U}_t^{cut})}(q)$.
	
	Theorem \ref{the:cut_independence} below tells  us that the actual state of a process given by a stage $u$ determines its immediate development regardless of the possible unfolding walk taken by a unit from the root position to a position $w \in u$. Furthermore, equation~\ref{eq:cut_independence_function} guarantees that this is the only conditional independence statement that can be read between an downstream and upstream variables ${X}(\mathcal{U}_t^{cut})$ and $Z(\mathcal{U}_t^{cut})$ measurable with respect to an $N\text{T-DCEG}$. 
	
	\begin{myTheorem}
		\label{the:cut_independence}
		Take a cut $\mathcal{U}_t^{cut}$, $t=-1,0,1,\ldots$, in an $N$T-DCEG $\mathbb{C}$. Then
		\begin{equation}
		{X}(\mathcal{U}_t^{cut})
		\independent
		{Z}(\mathcal{U}_t^{cut})
		|
		{Q}(\mathcal{U}_t^{cut}).
		\label{eq:cut_independence}
		\end{equation}
		Additionally, if a function $f({Z}(\mathcal{U}_t^{cut}))$ satisfies
		\vspace{-5pt}
		\begin{equation}
		{X}(\mathcal{U}_t^{cut})
		\independent
		{Z}(\mathcal{U}_t^{cut})
		|
		f({Z}(\mathcal{U}_t^{cut})),
		\label{eq:cut_independence_function}
		\vspace{-5pt}
		\end{equation}
		then ${Q}(\mathcal{U}_t^{cut})$ is a function of $f({Z}(\mathcal{U}_t^{cut}))$  with probability one. These results also hold when a cut $\mathcal{U}_T^{cut}$ is defined in a CEG $\mathbb{C}_t \in \mathcal{F}(\mathbb{C})$, $t=T,T+1,\ldots$.
	\end{myTheorem}
	\begin{proof}
		See \ref{app:cut_independence}.
	\end{proof}
	
	Analogously to the BN framework, these constructions now enable us to identify conditional independence structures embedded within an ${N\text{T-DCEG}}$ that hold for all values of conditioning variables. Despite the probability mass function of each variable associated with a cut~$\mathcal{U}_t^{cut}$ often being different over time~$t$, ${t=N-1,N,\ldots}$, the collection of conditional independence statements that can be read from them is nevertheless equivalent. This happens because by Definition~\ref{def:Cut} each cut~$\mathcal{U}_t^{cut}$, ${t=N-1,N,\ldots}$, corresponds to the same set of positions in an $N$T-DCEG~$\mathbb{C}$. This assertion is also valid for every CEG~$\mathbb{C}_t \in \mathcal{F}(\mathbb{C})$, ${t=N-1,N,\ldots}$, since each time-slice~$t$ in~$\mathbb{C}_t$ has the same stage structure as that of the subgraph $\mathbb{D}_H \subset \mathbb{C}$ (\cite[Theorem~5]{Collazo.Smith.2018a}). The concept of cut is illustrated in the example below. Note that some care is needed in reading conditional independences associated with the time-slice~${N\!-\!1}$ because a path in the initial graph~$\mathbb{D}_I$ may have probability zero. For an example about this case see~\cite[Section~7.7]{Collazo.2017} and~\cite{Collazo.Smith.2017}.
	
	\begin{myExampleCont1}
		\label{ex:cut}
		\emph{
			Recall the 2T-DCEG of Figure \ref{fig:2T-DCEG_Radicalisation_3var}. Take the cut $\mathcal{U}_t^{cut}=\{u_{13},u_{14}\}$ for $t=1,2,\ldots$. The variable ${X}(\mathcal{U}_t^{cut})$ then corresponds to the initial variable Transfer. The variable ${Q}(\mathcal{U}_t^{cut})$ whose state space is given by $\mathbb{Q}(\mathcal{U}_t^{cut})=\{1,2\}$, such that $\zeta_{Q(\mathcal{U}_t^{cut})}(1)=u_{13}$ and $\zeta_{Q(\mathcal{U}_t^{cut})}(2)=u_{14}$, provides us an reinterpretation of the initial variable Radicalisation~R. In this case the variable ${Q}(\mathcal{U}^{cut})$ tell us that the variable~R can be collected and analysed as a binary variable $R^*$ that identifies whether a prisoner has adopted radicalisation ($Q=2$) or not ($Q=1$). We then have that}
		%\vspace{-5pt}
		\begin{equation}
		T(t+1) \independent \boldsymbol{A}|(R^*(t+1),T(t)=n), \qquad t=0,1,\ldots,
		%\vspace{0pt}
		\end{equation}
		\emph{
			$\!$where $\boldsymbol{A}=(N(t+1),\ldots,N(0),T(t-1),\ldots,T(0),R(t),\ldots,R(0))$.
			Theorem~\ref{the:cut_independence} also guarantees that there is no information gain using the variable~R with three categories to predict the probability of an inmate to be transfer to another prison once we have observed the variable ${Q}(\mathcal{U}^{cut})$.}
	\end{myExampleCont1}
	
	A cut allows us to describe conditional independences concerning developments 1-step ahead of situations in a staged tree. However, if a unit's developments over the next $s$ time steps are of interest then we need to use an extended definition of fine cut to accommodate the time window~$s$ within which the present can affect the future. For this purpose, let $\mathcal{W}^{cut(s)}_t$ be the set of positions corresponding to $\mathcal{W}_t^{cut}$ when the focus is on $s$ time steps ahead from the actual time. These new definitions are particular important because now a fine cut provides us with a framework to identify global conditional independence structures that naturally arises from an $N\text{T-DCEG}$.
	
	Note that when the time window $s$ is greater than $N\!-\!2$ all necessary information to define a random variable associated with a fine cut can be obtained in a straightforward way from the $N\text{T-DCEG}$~$\mathbb{C}$: ${\mathcal{W}_t^{cut(s)}=\mathcal{W}_t^{cut}}$. On the other hand, for shorter~$s$ we also need to use the CEG ${\mathbb{C}_{g(t)+s} \in \mathcal{F}(\mathbb{C})}$ to define these variables, where $g$ is a function such as ${g(x)=x}$, if ${x \leq N\!-\!2}$, and ${g(x)=N\!-\!1}$, otherwise. This is because the position set $\mathcal{W}_t^{cut(s)}$ associated with the current time $t$ may be a coarser partition of situations than $\mathcal{W}_t^{cut}$ if $s$ time-slices unfold from time $t$, when ${s=0,\ldots,N\!-\!2}$. For more detail, see the discussion in~\cite{Collazo.Smith.2018a}.
	
	Let $\Lambda(\mathcal{W}_t^{cut})=\cup_{w \in \mathcal{W}_t^{cut}}\Lambda(w,t)$, where $\Lambda(w,t)$ denotes the set of all walks ${\lambda=(w_{0},\ldots,w{(t)}) \subset \mathbb{C}}$, such that each walk~$\lambda$ passes through cycle temporal edges exactly $h(t)$ times. Also let $\Lambda_{s}(w)$ be the set of all walks that unfolds from~$w$ over $s$~time-slices in $\mathbb{C}$ and $\xi(\lambda)$ be the sequence of events associated with a walk~$\lambda$. 
	Finally, let 
	$\Xi(\mathcal{W}_t^{cut(s)})=\cup_{w \in \mathcal{W}_t^{cut(s)}} \{\xi(\lambda);\lambda \in \Lambda_{s}(w)\}$ 
	denote the set of sequences of events~$\xi(\lambda)$ that can unfold from $\mathcal{W}_t^{cut}$ over $s$~time-slices. When $s$ is equal to zero, $\Xi(\mathcal{W}_t^{cut(s)})$ denotes the set of developments~$\xi(\lambda)$ from $\mathcal{W}_t^{cut}$ during the current time-slice~$t$. 
	Analogously to a cut, we can now define three useful random variables that assume values over time~$t$ and time step~$s$, $t=0,1,\ldots$ and $s={N\!-\!1},N,\ldots$, as follows:
	
	\begin{enumerate}[nosep]
		\item $X(\mathcal{W}_t^{cut(s)})$ is the downstream random variable of  $\mathcal{W}_t^{cut}$ in $\mathbb{C}$ whose state space is the set $\mathbb{X}(\mathcal{W}_t^{cut(s)})=\{1,2,\ldots,|\Xi(\mathcal{W}_t^{cut(s)})|\}$, such that there exists a bijection ${\varpi_{X(\mathcal{W}_t^{cut(s)})}:\mathbb{X}(\mathcal{W}_t^{cut(s)}) \to \Xi(\mathcal{W}_t^{cut(s)})}$. Its probability mass function $\pi_{X}(x)$ is defined by
		\vspace{-5pt}
		\begin{equation}
		\label{eq:X_fine_cut_mass_function}
		\pi_{X}(x) \propto 
		\sum_{\bar{\lambda} \in \Lambda_x(\mathcal{W}_t^{cut(s)})}
		\prod_{\substack{w \in \bar{\lambda} \\ w \neq l(\bar{\lambda})}} \pi(w'(w)|w),
		\quad x \in \mathbb{X}(\mathcal{W}_t^{cut(s)}),
		\vspace{-5pt}
		\end{equation}
		where ${\Lambda_x(\mathcal{W}_t^{cut(s)}) \!=\!
			\{\bar{\lambda}\!=\!(\lambda_*,\lambda) \subseteq \mathbb{C};\lambda_* \in \Lambda(\mathcal{W}_t^{cut}) \text{ and } \xi(\lambda)\!=\!\varpi_{X}(x)\}}$
		is the set of all walks $\bar{\lambda}$ in $\mathbb{C}$ that have two disjoint sub-walks ${\lambda_* \!\!\in\! \Lambda(\mathcal{W}_t^{cut})}$ and $\lambda$, such that $\lambda$ unfolds from $\lambda_*$ over~$s$ time-slices and $\xi(\lambda)\!=\!\varpi_{X}(x)$.
		
		\item ${Q}(\mathcal{W}_t^{cut(s)})$ is the separator random variable whose state space is given by the set ${\mathbb{Q}(\mathcal{W}_t^{cut(s)})=\{1,2,\ldots,|\mathcal{W}_t^{cut}|\}}$, such that there is a bijection $\varpi_{Q(\mathcal{W}_t^{cut})}:\mathbb{Q}(\mathcal{W}_t^{cut}) \to \mathcal{W}_t^{cut}$. Its probability mass function $\pi_{Q}(q)$ is proportional to the sum of all the monomials in primitives associated with $\Lambda(w,t), w \in \mathcal{W}_t^{cut}$. Symbolically then,
		\vspace{-5pt}
		\begin{equation}
		\label{eq:Q_fine_cut_mass_function}
		\pi_{Q}(q) \propto 
		\sum_{\lambda \in \Lambda(\varpi_Q(q),t)}
		\prod_{\substack{{w} \in \lambda \\ {w} \neq l(\lambda)}} 
		\pi(w'({w})|{w}), \quad 
		q \in \mathbb{Q}(\mathcal{W}_t^{cut}).
		\vspace{-5pt}
		\end{equation}
		
		\item $Z(\mathcal{W}_t^{cut(s)})$ is the upstream random variable of  $\mathcal{W}_t^{cut}$ in $\mathbb{C}$ whose state space consists of 
		the set ${\mathbb{Z}(\mathcal{W}_t^{cut})=\{1,2,\ldots,|\Lambda(\mathcal{W}_t^{cut})|\}}$, such that there is a bijection $\varpi_{Z(\mathcal{W}_t^{cut})}:\mathbb{Z}(\mathcal{W}_t^{cut}) \to \Lambda(\mathcal{W}_t^{cut})$. Its probability mass function is proportional to each monomial in the primitives corresponding to a walk that constitutes its state spaces. Explicitly,
		\begin{equation}
		\label{eq:Z_fine_cut_mass_function}
		\pi_{Z}(z) 
		\propto 
		\prod_{\substack{{w} \in \lambda \\ {w} \neq l(\lambda)}}
		\pi(w'({w})|{w}), \quad 
		z \in \mathbb{Z}(\Lambda(\mathcal{W}_t^{cut})).
		\end{equation}
		where $\lambda=\varpi_{Z(\mathcal{W}_t^{cut})}(z)$.
	\end{enumerate}
	Observe that again `$=$' can substitute `$\propto$' in the equations above if the $N\text{T-DCEG}$ does not have a sink position.
	
	To define these three variables when $t\!=\!0,1,\ldots$ and $s\!=\!0,\ldots,N\!-\!2$, take a partition~${\beth_t^s=\{\beth_{t,1}^s,\ldots,\beth_{t,K}^s\}}$ of $\mathcal{W}_t^{cut}$. Recall from~\cite{Collazo.Smith.2018a} that the position structure of $\mathbb{C}_{g(t)+s}$ at time~$g(t)$ results from an application of the vertex contraction operator and so naturally yields a partition $\beth_t^{s}\!=\!\{\beth_{t,1}^{s},\ldots,\beth_{t,K}^{s}\}$ over~$\mathcal{W}_t^{cut}$ according to the merged vertices. Now set 
	${\mathcal{W}_t^{cut(s)}=\beth_t^s}$. For all $t=-1,0,\ldots$, the definitions of random variables ${X}(\mathcal{W}_t^{cut(s)})$ and ${Z}(\mathcal{W}_t^{cut(s)})$, ${s=0,\ldots,N\!-\!2}$, are identical to the variable $X$ and $Z$ associated with $\mathcal{W}_t^{cut(s)}$, ${s=N\!-\!1,N,\ldots}$. It then follows that equations~\ref{eq:X_fine_cut_mass_function} and~\ref{eq:Z_fine_cut_mass_function} remain valid when $s=0,\ldots,N\!-\!2$. 
	
	Of course, the variable $Q(\mathcal{W}_t^{cut(s)})$ has to be redefined appropriately since its state space is now given by ${\mathbb{Q}(\mathcal{W}_t^{cut(s)})=\{1,2,\ldots,|\beth_t^{s}|\}}$, , such that there is a bijection $\varpi_{Q(\mathcal{W}_t^{cut(s)})}:\mathbb{Q}(\mathcal{W}_t^{cut(s)}) \to \beth_t^{s}$. Its probability mass function~$\pi_{Q}(q)$ then corresponds to the weighted sum of all probability masses defined by equation~\ref{eq:Q_fine_cut_mass_function} associated of a position in $\beth_{t,i}^{s}$. Symbolically we therefore have that
	\begin{equation}
	\label{eq:Q_fine_cut_s_mass_function}
	\pi_{Q}(q) \propto
	\sum_{w \in \varpi_Q(q)}
	\sum_{\lambda \in \Lambda(w,t)}
	\prod_{\substack{\bar{w} \in \lambda \\ \bar{w} \neq l(\lambda)}} 
	\pi(w'(\bar{w})|\bar{w}),
	\quad q \in \mathbb{Q}(\mathcal{W}_t^{cut(s)}).
	\vspace{-5pt}
	\end{equation}
	
	These random variables enable us to read a large collection of conditional independence statements between vectors of functions of primitive random variables embedded into the $N\text{T-DCEG}$ topology. This is because a fine cut is based on positions that gather situations in a staged tree all of whose future developments are equivalent. Theorem~\ref{the:fine_cut_independence} tells us that a unit's future unfoldings  are independent from the whole set of its past events given that the available information on it constitutes a fine cut. It also guarantees that, given a fine cut at time~$t$ and time-horizon~$s$, a function of upstream variables that makes all the corresponding downstream variables conditionally independent from upstream variables must constitute a fine cut.      
	
	\begin{myTheorem}
		\label{the:fine_cut_independence}
		Take a fine cut $\mathcal{W}_t^{cut}$, $t=-1,0,1,\ldots$, in an $N$T-DCEG $\mathbb{C}$. For every ${s=0,1,\ldots}$, we have that
		\begin{equation}
		\boldsymbol{X}(\mathcal{W}_t^{cut(s)})
		\independent
		\boldsymbol{Z}(\mathcal{W}_t^{cut(s)})
		|
		\boldsymbol{Q}(\mathcal{W}_t^{cut(s)}).
		\label{eq:fine_cut_independence}
		\end{equation}
		Additionally, if a function $f({Z}(\mathcal{W}_t^{cut(s)}))$ satisfies
		\vspace{-5pt}
		\begin{equation}
		{X}(\mathcal{W}_t^{cut(s)})
		\independent
		{Z}(\mathcal{W}_t^{cut(s)})
		|
		f({Z}(\mathcal{W}_t^{cut(s)})),
		\label{eq:fine_cut_independence_function}
		\vspace{-5pt}
		\end{equation}
		then ${Q}(\mathcal{W}_t^{cut(s)})$ is a function of $f({Z}(\mathcal{W}_t^{cut(s)}))$  with probability one. These results also hold when a fine cut $\mathcal{W}_T^{cut(s)}$ is defined in a CEG $\mathbb{C}_{t+s} \in \mathcal{F}(\mathbb{C})$, ${t=T,T+1,\ldots}$.
	\end{myTheorem}
	\begin{proof}
		See \ref{app:fine_cut_independence}.
	\end{proof}
	
	Thus, the units' behaviours may present important differences in the medium and long time (${s \!\geq\! N\!-\!1}$) but may be undistinguishable in the short term (${s\!\leq\! N\!-\!2}$). For analogous reasons to those discussed for a cut, a fine cut 
	$\mathcal{W}_t^{cut(s)}$ entails the same set of conditional probability statements  when $t=N-1,N,\ldots$ for a given time-horizon~$s$ and when $s=N-1,N,\ldots$ for a given time-slice~$t$. Thus, for any 2T-DCEG $\mathbb{C}$ all these constructions associated with a fine cut only require the three CEGs $\mathbb{C}_{i},i=0,\ldots,2$. 
	
	Recall that Theorem~\ref{the:cut_independence} tells us that a cut gathers all the necessary information to predict the immediately development of a unit in a process. Theorem~\ref{the:fine_cut_independence} guarantees that the future events are independent from past events given the actual state of a process. Therefore, these results enable us to use cuts and fine cuts to deduce some conditional independence statements given some observed effects in a way that extends the BN and DBN framework using the d-separation theorem. This happens because a Stratified DCEG can enrich the conditional independence hypotheses depicted in its corresponding DBN with context-specific deductions and the time horizon~$s$. These links are further discussed through the example below using the concept of fine cut.
	
	\begin{myExampleCont1}
		\emph{
			Return to the 2T-DCEG~$\mathbb{C}$ depicted in Figure~\ref{fig:2T-DCEG_Radicalisation_3var}. Assume that domain experts are interested in exploring the impacts of social connections on the risk of inmate's radicalisation and transfer at the current time-slice $t,t=1,2,\ldots$, given that the past information is completely available.
			To do this, we can use the fine cut
			\vspace{-5pt}
			$$\mathcal{W}_t^{cut}\!\!=\!\!\{w_{10},w_{11},w_{12},w_{19},w_{20},w_{21}\}.\vspace{-5pt}$$
			Since the experts' focus is on the current time-slice, we also need to set the time window~$s$ equal to $0$. As discussed previously, this time horizon then requires us to replace the fine cut~$\mathcal{W}_t^{cut}$ by  $\mathcal{W}_t^{cut(0)}=\beth_t^0$. Using the CEG $\mathbb{C}_2 \in \mathcal{F}(\mathbb{C})$ showed in Figure~\ref{fig:2T-DCEG_as_CEG}, we can see that $\beth_t^0=\{\beth_{t,i}^0,i=1,\ldots,4\}$, where $\beth_{t,1}^0=\{w_{10}\}$, $\beth_{t,2}^0=\{w_{11}\}$, $\beth_{t,2}^0=\{w_{12}\}$ and $\beth_{t,4}^0=\{w_{19},w_{20},w_{21}\}$.}
		
		\emph{
			The random variable $Q(\mathcal{W}_t^{cut(0)})$ then has four states $\{1,\ldots,4\}$, such that ${\varpi_{Q(\mathcal{W}_t^{cut(0)})}(i)=\beth_{t,i}^0}$, $i=1,\ldots,4$. These have the following interpretations: categories $1,2$ and $3$ characterise  non-adopting prisoners whose social networks were classified, respectively, as Sporadic, Frequent and Intense at time~$t-1$; and category~4 represents a prisoner adopting radicalisation at time~$t-1$. The variable $X(\mathcal{W}_t^{cut(0)})$ has $24$ states associated bijectively with the set of sequences of events
			\vspace{-5pt} $${\Xi(\mathcal{W}_t^{cut(0)})\!=\!\{({R},T),(N,R,T);N\!=\!s,f,i,{R}\!=\!r,v,a \text{ and } T\!=\!n,t\}}.\vspace{-5pt}$$
			So given a fine $\mathcal{W}_t^{cut(0)}$ we have $24$ possible outcomes at the end of this time-slice.}
		
		%\vspace{19pt}
		\begin{figure}[ht] 
			\begin{center}
				\includegraphics[scale=0.69,angle=-90,origin=c,trim=50 17 30 -130]{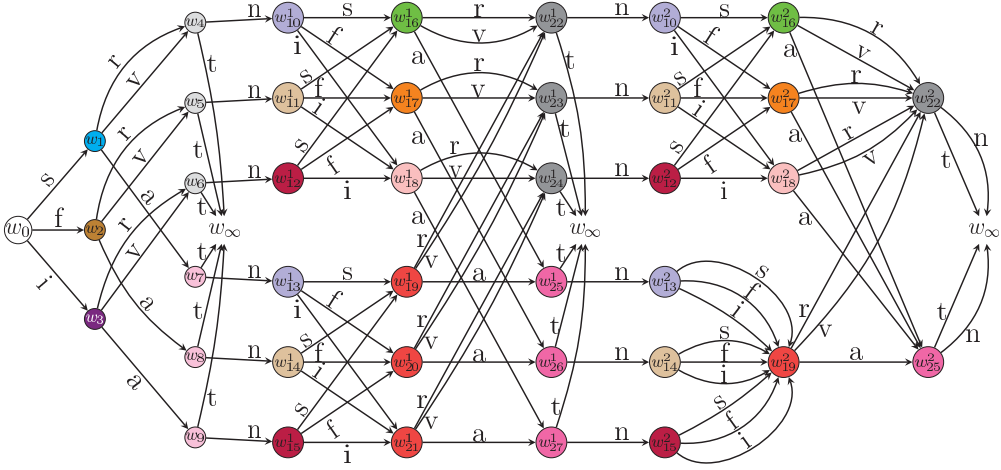}
			\end{center}
			\vspace{-90pt}
			\caption{The CEG $\mathbb{C}_2$ associated with the 2T-DCEG depicted in Figure \ref{fig:2T-DCEG_Radicalisation_3var} \label{fig:2T-DCEG_as_CEG}}
			\vspace{-5pt}
		\end{figure}
		
		\emph{
			Additionally this fine cut $\mathcal{W}_t^{cut(0)}$ tells us that the social network of an adopting prisoner at time~$t$ affects neither his radicalisation process nor his transfer probability. However this is not true if the inmate is not radicalised. Note that this kind of context-specific d-separation statement cannot be directly deducted from a BN in Figure~\ref{fig:DBN_Radicalisation} or its corresponding undirected moralised graph.
			For a larger time length, $s=1,2,\ldots$, we can see directly from the 2T-DCEG (Figure~\ref{fig:2T-DCEG_Radicalisation_3var}) that the future developments of a prisoner depend on the random vector $\boldsymbol{Q}(\mathcal{W}_t^{cut(s)})$ whose set of states are given by $\mathcal{W}_t^{cut}$. So now the current status of an adopting prisoner's network has an impact in his unfolding events.}
		
		\emph{
			Note that analogous conclusions could have been obtained if we had used the fine cut  $\mathcal{W}_t^{cut}\!\!=\!\!\{w_{16},\ldots,w_{21}\}$. However, in this case the interpretation of the variables would be based on the actual (time~$t$) social network of an inmate instead of his previous social classification at time~$t-1$. Also observe that the distribution of the variable $Q(\mathcal{W}_t^{cut(0)})$ would not depend on any information at time~${t}$ about the radicalisation status of an inmate or his chance of being transferred. This would enable us to easily update our judgements at time~$t$ as new information about the social contacts of an inmate is collected.}
	\end{myExampleCont1}
	
%%%%%%%%%%%%%%%%%%%%%%%%%%%%%%%%%%%%%%%%%%%%%%%%%%%%%%%%%%%%%%
\section{Local independence and Granger noncausality}
\label{sec:local_global_independence}

Schweder \cite{Schweder.1970} first introduced the concept of local independence for Markov processes. Subsequently Aalen \cite{Aalen.1987} generalised this and applied it to processes with a Doob-Meyer decomposition. Didelez \cite{Didelez.2008} then further extended the concept so that it applied to general multivariate processes. The notion of local independence is useful because in a model that fully represents a process a local independence statement can be translated into Granger noncausality (\cite{Granger.1969}); see e.g. \cite{Eichler.2007,Didelez.2008}. Granger noncausality has recently also been discussed for mediation and intervention \cite{Eichler.Didelez.2010,Aalen.2012}. Here we develop the idea of local independence for the discrete time processes expressed within a DCEG. For technical consistency, if there is a terminating event in a DCEG, the following concepts of conditional independences are valid for a unit that experiences a terminating event at time $T$ as long as ${t=0,\ldots,T\!-\!1}$.  

Consider two random variables $X$ and $Y$ that take value over each time-slice. By saying that $X$ is locally independent from $Y$ we mean that the past values of $Y$ do not provide any additional information to predict the current value of $X$ given all set of past events up to the current time. Note that in a DCEG model these variables do not need to begin happening at the initial time-slice $t=0$; they can start to happen later. Also observe that in some DCEGs we may be interested only in time-slices from a certain time $T$ on. This is often the case for an $N\text{T-DCEG}$ model where the experts tend to focus on its cyclic subgraph $\mathbb{D}_H$. 

To handle these cases, building on previous work by \cite{Eichler.2007,Didelez.2008,Eichler.Didelez.2010}
we introduce the concept of $T$-local independence below. This enables us to analyse the impact of past events on the current target process from a time-slice $T$ on. This idea directly generalises to random vectors $\boldsymbol{X}\!$ and $\boldsymbol{Y}\!$. Let $\mathcal{E}^{(t)}$ denote the collection of all sequences of events that happened up to the end of time-slice $t$ and let $\mathcal{E}^{(t-1)}_{(-\boldsymbol{X})} \!\!\subset\! \mathcal{E}^{(t)}$ denote the history of past events that excludes information with respect a random vector~$\boldsymbol{X}$. 

\begin{myDef}
\label{def:Local_independence}
Take two random vectors $\boldsymbol{X}$ and $\boldsymbol{Y}$ measurable with respect every time-slice $t, t \geq T$, of a DCEG. A vector $\boldsymbol{X}$ is said to be $\boldsymbol{T}$\textbf{-locally independent} from $\boldsymbol{Y}$ if all probability distributions $p_{\boldsymbol{X}(t)}(\boldsymbol{x}(t)|\mathcal{E}^{(t-1)})$ are measurable with respect to $\mathcal{E}^{(t-1)}_{(-\boldsymbol{Y})}$ for all $t=T,T+1,\ldots$. Denote this by $\boldsymbol{X} {\underrightarrow{ \independent}_{T}} \boldsymbol{Y}$. If the local independence condition holds for all time-slice $t, {t\geq 0}$, then $\boldsymbol{X}$ is said to be locally independent from $\boldsymbol{Y}$. Henceforth we will denote local independence by $\boldsymbol{X} \underrightarrow{ \independent} \boldsymbol{Y}$.
\end{myDef}

Assuming that the underlying event tree of a DCEG $\mathbb{C}$ completely describes the natural behaviour of a process, a random vector $\boldsymbol{Y}$ is (strongly) Granger noncausal for $\boldsymbol{X}$ with respect to $\mathbb{C}$ if $\boldsymbol{X} \underrightarrow{ \independent} \boldsymbol{Y}$. Otherwise, we say that $\boldsymbol{Y}$ is Granger causal or a \emph{prima facie} cause for $\boldsymbol{X}$. Analogously, we say that a random vector $\boldsymbol{Y}$ is $T$-Granger noncausal for $\boldsymbol{X}$ with respect to $\mathbb{C}$ if $\boldsymbol{X} \underrightarrow{ \independent}_T \boldsymbol{Y}$. For the validity of the Granger causal interpretation, events not depicted in the event tree cannot span Granger causal structures between random variables measurable with respect the corresponding DCEG.  

The $T$-local independence relation is not necessarily symmetric and therefore neither is Granger noncausality \cite{Didelez.2008}. For example, in the 2T-DCEG depicted in Figure \ref{fig:2T-DCEG_Radicalisation_3var} the network variable $N$ is locally independent from the radicalisation variable $R$ but the inverse relation does not hold. So under the assumption that this model is a fair representation of the radicalisation process in a prison we can say that $R$ is Granger noncausal for $N$ whilst $N$ is a \emph{prima facie} cause for $R$.

In discrete time we are often also interested in exploring intra-time conditional independences that characterize each time-slice given the whole set of past events. This differentiates the DCEG from a continuous time graphical model \cite{Gottard.2007,Didelez.2008} where two different counting processes cannot represent the same event. In those frameworks, a prisoner is assumed not to be radicalised and to change his social network at the same time.

In this respect the DCEG models come closest to the path diagrams used to visualise the dynamic of multivariate weakly stationary multivariate time series \cite{Eichler.2007}. However path diagrams have a different graphical semantic from DCEG models because there vertices represent processes, directed edges correspond to local dependences and dashed edges depict intra-time dependences. This makes them unable to represent any graphically context-specific hypotheses. In these models all time-slices also have the same conditional independence structure. For the purpose of this paper it is therefore useful to introduce the concept of $T$-contemporaneous independence. Here we follow some previous authors~\cite{Granger.1980,Eichler.2007,Eichler.Didelez.2010}.

\begin{myDef}
\label{def:Contemporaneous_Independence}
Take two random vectors $\boldsymbol{X}$ and $\boldsymbol{Y}$ measurable with respect to every time-slice $t, t \geq T$, of a DCEG. These variables are said to be $\boldsymbol{T}\text{\textbf{-contemporaneously independent}}$ if for every $t, t=T, T+1,\ldots$, their joint probability distribution is such that  
\begin{equation}
\label{eq:simultaneous_independence}
p_{\boldsymbol{X}(t),\boldsymbol{Y}(t)}
	(\boldsymbol{x}(t),\boldsymbol{y}(t)|\mathcal{E}^{(t-1)})=
							p_{\boldsymbol{X}(t)}(\boldsymbol{x}(t)|\mathcal{E}^{(t-1)})
							p_{\boldsymbol{Y}(t)}(\boldsymbol{y}(t)|\mathcal{E}^{(t-1)}).
\end{equation}
This will be denoted by $\boldsymbol{X} \stackrel{\independent_{\!T}}{\sim} \boldsymbol{Y}|\mathcal{E}^{(t-1)}$. If this property holds for all time-slices, the variables are simply said to be contemporaneously independent and the subscript $T$ can be dropped from the notation.  
\end{myDef}

Finally the stochastic independence given in Definition~\ref{def:Stochastic_Independent} below establishes the condition for two random vectors to be globally independent given the past events. Theorem~\ref{pro:stochastic_local_independence} guarantees that this kind of stochastic independence only happens in the presence of contemporaneous and local independences. This result provides us with a framework that enables us to determine whether or not two variables are stochastically independent without verifying equation~\ref{eq:stochastic_independence} using various algebraic calculations. For example, from Figure~\ref{fig:2T-DCEG_Radicalisation_3var} we can read directly from this ${2\text{T-DCEG}}$ that the variables $N, R$ and $T$ are not stochastically independent since they are not contemporaneously independent. 

\begin{myDef}
\label{def:Stochastic_Independent}
Take two random vectors $\boldsymbol{X}$ and $\boldsymbol{Y}$ measurable with respect every time-slice $t, t \geq T$, of a DCEG. These variables are $\boldsymbol{T}$\textbf{-stochastically independent} if for every $t, t=T, T+1,\ldots$, their joint probability distribution is such that  
\begin{equation}
\label{eq:stochastic_independence}
p_{\boldsymbol{X}(t),\boldsymbol{Y}(t)}
(\boldsymbol{x}(t),\boldsymbol{y}(t)|\mathcal{E}^{(t-1)})=
		p_{\boldsymbol{X}(t)}(\boldsymbol{x}(t)|\mathcal{E}^{(t-1)}_{(-\boldsymbol{Y})})
		p_{\boldsymbol{Y}(t)}(\boldsymbol{y}(t)|\mathcal{E}^{(t-1)}_{(-\boldsymbol{X})}).
\end{equation}
This is denoted by $\boldsymbol{X} {\independent_{\!T}} \boldsymbol{Y}$. If at least one of these variables does not exist in any time-slice $t<T$, they are stochastically independent. We then simply denote this by writing $\boldsymbol{X} \independent \boldsymbol{Y}$.
\end{myDef}

\begin{myTheorem}
\label{pro:stochastic_local_independence}
Two random variables $\boldsymbol{X}$ and $\boldsymbol{Y}$ measurable with respect a DCEG are $T$-stochastically independent if and only if they are mutually $T\text{-locally}$ independent and $T$-contemporaneously independent.
\end{myTheorem}
\begin{proof}
See \ref{app:local_stochast_simult_independence}.
\end{proof}

%%%%%%%%%%%%%%%%%%%%%%%%%%%%%%%%%%%%%%%%%%%%%%%%%%%%%%%%%%%%%%%
\section{Conclusions}
\label{sec:Conclusion}

An $N$T-DCEG provide us with a rich structure to incorporate information extracted from domain experts and a data set. In Section~\ref{sec:Composite_Model}  
we have proposed an algorithmic tool based on the $N$T-DCEG framework to support rational consensus building~\cite{Cooke.1991} between domain experts, modellers and decision makers. We have also shown that being able to graphically depict context-specific conditional independence statements the 2T-DCEG models contain all discrete $2\text{T-DBN}$ models as a special case. We have argued that this fact can be used to design efficient model selection algorithm that make good use of computational time and memory to search the $2$T-DCEG model space.

We have developed a methodology to construct non-trivial random variables from an $N$T-DCEG topology. Particularly for a process with highly asymmetric developments this facilitates the translation of an $N$T-DCEG model into a $N$T-DBN model and stimulates the exploration of further links between these two model classes. The systemic construct of random variables can also be used for inductive reasoning and abductive reasoning based, respectively, on cuts and fine cuts.      

It was demonstrated that the conditional independences in an ${N\text{T-DCEG}}$ can also be interpreted in terms of Granger noncausality. However this relationships needs to be developed further. For example, by identifying a fine cut $\mathcal{W}^{cut(s)}$ we are able to expose the conditional independences across different time steps~$s$. This then leads us to a direct link with Granger noncausality as it applies to different time horizons~\cite{Dufour.Renault.1998} and gives us a new graphical framework which appears to be able to distinguish between short-run and long-run causal mechanisms.

Here the Granger noncausality is defined with respect to the whole set of past events $\mathcal{E}^{(t)}$. However, we may want to focus our attention on a particular subset of events $\mathcal{E}^{(t)}_*\subset \mathcal{E}^{(t)}$. Similarly, the conditions determining the Granger noncausal relations with respect to a proper subset of events $\mathcal{E}_*^{(t)}$  in a DCEG still remain unexplored as are the types of assumptions about local independences that are needed in order to assess the causal effect arising from an intervention on the system. As discussed for path diagrams applied to time series \cite{Eichler.Didelez.2010}, these developments demand the combination of the Granger noncausality idea \cite{Granger.1969} and the Pearl's stronger causality concept~\cite{Pearl.2009} presented in \cite{Thwaites.etal.2010,Thwaites.2013}. An ${N\text{T-DCEG}}$ also appears to provide a useful framework for deriving contemporaneous causal relations \cite{Granger.1988, Lutkepohl.1990}.

%%%%%%%%%%%%%%%%%%%%%%%%%%%%%%%%%%%%%%%%%%%%%%%%%%%%%%%%%%%%%%%
\appendix
%%%%%%%%%%%%%%%%%%%%%%%% 
\section{Proof of Theorem \ref{the:DBN_SDCEG}}
\label{app:DBN_SCDEG}

\noindent
\emph{
	All conditional independence statements entailed by the ordered Markov property  in a DBN can be depicted by an SDCEG.
	\vspace{3pt}}

Take a DBN model defined by a set of random variables~$\boldsymbol{\mathcal{Z}}=\{\mathcal{Z}_1,\ldots,\mathcal{Z}_n\}$, a permutation~$I=(i_1,\ldots,i_n)$ and a DAG~$\mathbb{D}=(\bigcup_t V(t), \bigcup_t E(t) \cup E_\dagger(t)$) corresponding to~$\boldsymbol{Z}(I)$, or simply~$\boldsymbol{Z}$. Denote by~$\mathbb{Z}_{i_k}$ the set of events corresponding to the values of~$\mathcal{Z}_{i_k}$.  Let ${\mathbb{Z}^{(k)}(I)=\mathbb{Z}_{i_{1}}\times \mathbb{Z}_{i_{2}}\times \cdots \times \mathbb{Z}_{i_{k}}}$
be the product space of the first $k$ variables in $\boldsymbol{Z}(I)$. Denote by $(v_i,v_j,z)$ a labelled edge in an event tree, where $z$ is the label of a directed edge~$(v_i,v_j)$.

Now construct the finite event tree~$\mathcal{T}(\boldsymbol{Z})=(V(\boldsymbol{Z}),E(\boldsymbol{Z}))$ as follows:
	\begin{enumerate}[nosep,leftmargin=0.9cm]
		\item The situation set $V(\boldsymbol{Z})$ is formed by a root situation $s_{0}$ together with a set of nodes~$v(\boldsymbol{z}^{(k)})$, one for
		each $\boldsymbol{z}^{(k)}=(z_{i_{1}},z_{i_{2}},\ldots ,z_{i_{k}})$ in $\mathbb{Z}^{(k)}(I)$, $k=1,\ldots,n$, such that, for $k=2,\ldots,n$, none of $z_{i_1}$,\ldots, $z_{i_{k-1}}$ correspond to a terminating event. Note that each $v(\boldsymbol{z}^{(n)})$,$v(\boldsymbol{z}^{(n)}) \in V(\boldsymbol{Z})$, is a leaf and so is every $v(\boldsymbol{z}^{k})$, such that $z_{i_{k}}$ is the only label in~$\boldsymbol{z}^{k}$ associated with a terminating event.
		\item The edge set $E(\boldsymbol{Z})$ consists of the set of labelled edges $(s_0,v(\boldsymbol{z}^{(1)}),z_{i_1})$, where $\boldsymbol{z}^{(1)}=(z_{i_1})$ and ${z_{i_1} \in  \mathbb{Z}_{i_1}}$, together with a set of labelled edges $(v(\boldsymbol{z}^{(k)}),v(\boldsymbol{z}^{(k+1)}),z_{i_{k+1}})$, $k=1,\ldots,n\!-\!1$, where: $v(\boldsymbol{z}^{(k)})$ and $v(\boldsymbol{z}^{(k+1)})$ are in~$V(\boldsymbol{Z})$; $z_{i_{k+1}} \in \mathbb{Z}_{i_{k+1}}$; and ${\boldsymbol{z}^{(k+1)}=(\boldsymbol{z}^{(k)},z_{i_{k+1}})}$. 
	\end{enumerate}
Observe that in the event tree $\mathcal{T}(\boldsymbol{Z})$ all of its non-root situations~$v(\boldsymbol{z} ^{(k)})$ are at the same distance~$k$. The root situation~$s_0$ corresponds to the variable~$\mathcal{Z}_1$ and each situation~$v(\boldsymbol{z} ^{(k)})$ in $V(\boldsymbol{Z})$ is associated with the variable~$\mathcal{Z}_{k+1}$. Let $\mathcal{T}_\infty(\boldsymbol{Z})$ be the infinite tree spanned by  TOG($\emptyset,\mathcal{T}(\boldsymbol{Z})$). 

Note that given the ordered Markov property (OMP) any vertex $v_{i_k}(t)$ in $\mathbb{D}$, which corresponds to variable~$\mathcal{Z}_{i_k}$ at time-slice~$t$, can be well-defined in the whole vertex set $\bigcup_t V(t)$ by a new index~$l$, $l=i_k + n*t-1$; i.e., $v_l \equiv v_{i_k}(t)$. Define the set~$U_1=\{s_0\}$. For every $L, L=1,2,\ldots$, now construct a partition of the situation set of~$\mathcal{T}_\infty(\boldsymbol{Z})$ using the DAG~$\mathbb{D}$ as follows:
\begin{enumerate}[nosep]	
\item Define the sequence $I_L=(l_1,\ldots,l_{n_L}), l_1<\ldots<l_{n_L}$, constituted by the indices~$l_i$ of all vertices in~$pa(\mathcal{Z}_L)$ with respect to~$\mathbb{D}$. 

\item Take the set of vectors $R_L=\{\boldsymbol{\rho}; \boldsymbol{\rho} \in \mathbb{Z}_{l_{1}} \times \cdots \times \mathbb{Z}_{l_{n_L}}\}$.

\item For every $\boldsymbol{\rho}=(\rho_1,\ldots,\rho_{l_{n_L}})$ in $R_L$, define the situation set $U_\rho^L$  constituted by all situations~$s_j$ at distance~$L$ from~$s_0$ in~$\mathcal{T}_\infty(\boldsymbol{Z})$, such that along the root-to-$s_j$ path the unfolding event of each situation~$s_{l_i}$, $l_i \in I_L$,  is~$\rho_{l_i}$. Let $U_L=\{U_\rho^L\}_{\rho \in R_L} $. If $R_L=\emptyset$, then $U_L$ is the set of all situations at distance~$L$ from~$s_0$ in~$\mathcal{T}_\infty(\boldsymbol{Z})$. 

\item If there is a vertex $v_l, l<L$, in~$\mathbb{D}$ such that $v_l$ and $v_L$ at different time-slices~$t_l$ and~$t_L$ represent the same variable $\mathcal{Z}_{i_k}$ whose conditional probability table is the same for~$t_l$ and~$t_L$, then for every~$\rho$ in $R_L$ do $U_\rho^l \gets U_\rho^l \cup U_\rho^L$ and $U_L \gets \emptyset$.  
\end{enumerate}

Define the stage structure associated with~$\mathcal{T}_\infty(\boldsymbol{Z})$ as $U=\bigcup_L U_L$ and obtain the corresponding $\boldsymbol{\mathcal{Z}}$-DCEG~$\mathbb{C}$. Observe that this DCEG is stratified since it is yielded by a TOG and each stage only merges situations corresponding to the same variable. Remember that by construction of~$\mathcal{T}_\infty(\boldsymbol{Z})$ all situations at the same distance from~$s_0$ are associated with the same variable. Also note that $\mathcal{T}_\infty(\boldsymbol{Z})$ is built using the ordering set by the OMP and each stage associated with a variable~$Z_{i_k}$ at time-slice~$t$ is defined according to the conditional independence statements given by the parent set of~~$Z_{i_k}$ at time-slice~$t$. Thus, by construction of~$\mathcal{T}_\infty(\boldsymbol{Z})$ and~$U$ the DCEG~$\mathbb{C}$ also represents the collection of all conditional independences represented by the OMP based on~$\mathbb{D}$.

%%%%%%%%%%%%%%%%%%%%%%%% 
\section{Proof of Corollary~\ref{the:2TDBN_2TSDCEG}}
\label{app:2TDBN_2TSCDEG}

\noindent
\emph{
	Every conditional independence statements defined by the ordered Markov property  in a $\text{$N\!$T-DBN}$ can be expressed in a \text{$N\!$T-SDCEG}.
	\vspace{3pt}}

Take a $N$T-DBN model defined by a set of random variables~$\boldsymbol{\mathcal{Z}}$, a permutation~$I$ and a DAG~$\mathbb{D}$ corresponding to~$\boldsymbol{Z}(I)$, or simply~$\boldsymbol{Z}$. Elicit the infinite event tree~$\mathcal{T}_\infty(\boldsymbol{Z})$ and the corresponding stage structure~$U$ as described in~\ref{app:DBN_SCDEG}. Thus, $\mathcal{T}_\infty(\boldsymbol{Z})$ is a TOG($\emptyset,\mathcal{T}(\boldsymbol{Z})$) and $U$ represents all conditional independence statements described by the ordered Markov property associated with~$\mathbb{D}$. Note that U also satisfies condition~$2$ in Defintion~\ref{def:SDCEG}. Since any $N$T-DBN model assumes a Markov condition of order~${(N\!-\!1)}$, the staged tree~$\mathcal{ST}_\infty(\boldsymbol{Z},U)$ yielded by~$\mathcal{T}_\infty(\boldsymbol{Z})$ and~$U$ is time-homogeneous after time~${(N\!-\!1)}$. Therefore, the \text{$N\!$T$\boldsymbol{\mathcal{Z}}$-SDCEG} supported by~$\mathcal{ST}_\infty(\boldsymbol{Z},U)$ and defined in terms of ${(N\!-\!1)\text{-positions}}$ satisfies Corollary~\ref{the:2TDBN_2TSDCEG}.  

%%%%%%%%%%%%%%%%%%%%%%%% 
\section{Proof of Theorem \ref{the:cut_independence}}
\label{app:cut_independence}

\noindent
\emph{
	Take a cut $\mathcal{U}_t^{cut}$, $t=-1,0,1,\ldots$, in an $N$T-DCEG $\mathbb{C}$. Then
	\vspace{-9pt}
	\begin{equation*}
	{X}(\mathcal{U}_t^{cut})
	\independent
	{Z}(\mathcal{U}_t^{cut})
	|
	{Q}(\mathcal{U}_t^{cut}).
	%\label{eq:cut_independence_proof}
	\vspace{-9pt}
	\end{equation*}
	Additionally, if a function $f({Z}(\mathcal{U}_t^{cut}))$ satisfies
	\vspace{-9pt}
	\begin{equation*}
	{X}(\mathcal{U}_t^{cut})
	\independent
	{Z}(\mathcal{U}_t^{cut})
	|
	f({Z}(\mathcal{U}_t^{cut})),
	%\label{eq:cut_independence_function_proof}
	\vspace{-9pt}
	\end{equation*}
	then ${Q}(\mathcal{U}_t^{cut})$ is a function of $f({Z}(\mathcal{U}_t^{cut}))$  with probability one. These results also hold when a cut $\mathcal{U}_T^{cut}$ is defined in a CEG $\mathbb{C}_t \in \mathcal{F}(\mathbb{C})$, $t=T,T+1,\ldots$.
	\vspace{3pt}}

By definition, if the value of $Q(\mathcal{U}_t^{cut})$ is observed, for example $q$, then any random variable based on a stage $\zeta_{Q(\mathcal{U}_t^{cut})}(q)$ is completely defined. So none of the $w_0$-to-$w(t)$ walks, $w(t) \in \zeta_{Q(\mathcal{U}_t^{cut})}(q)$,  can bring any additional information on the realization of $X(\mathcal{U}_t^{cut})$. 
Thus, $X(\mathcal{U}_t^{cut}) \independent Q(\mathcal{U}_t^{cut})|Z(\mathcal{U}_t^{cut})$.

Assume that $Q(\mathcal{U}_t^{cut})$ is not a function of $f(Z(\mathcal{U}_t^{cut}))$ with probability one. Then there are at least two non-zero probability walks $\lambda_1$ and $\lambda_2$ in $\Lambda(\mathcal{U}_t^{cut})$ such that $l(\lambda_1)$ and $l(\lambda_2)$ are in two different stages, $f(z_1)=f(z_2)$ and
\vspace{-5pt}
$$X(\mathcal{U}_t^{cut})|[Z(\mathcal{U}_t^{cut})=z_1] \not\equiv X(\mathcal{U}_t^{cut})|[Z(\mathcal{U}_t^{cut})=z_2],\vspace{-5pt}$$
where $z_1=\zeta_{Z(\mathcal{U}_t^{cut})}^{-1}(\lambda_1)$ and $z_2=\zeta_{Z(\mathcal{U}_t^{cut})}^{-1}(\lambda_2)$.
Thus, this would imply that $X(\mathcal{U}_t^{cut})$ and $Z(\mathcal{U}_t^{cut})$ are not conditionally independent given $f(Z(\mathcal{U}_t^{cut}))$, giving a contraction.

Finally, \cite[Theorem~$5$]{Collazo.Smith.2018a}~guarantees that a cut $\mathcal{U}_T^{cut}$ in an $N\text{T-DCEG}$~$\mathbb{C}$ also defines a cut at time $T$ in every CEG $\mathbb{C}_{t} \subset \mathcal{F}(\mathbb{C}), t=T, T+1,\ldots$ and by definition equations~\ref{eq:X_cut_mass_function}, \ref{eq:Q_cut_mass_function} and~\ref{eq:Z_cut_mass_function} remain valid. In this case, a cut $\mathcal{U}_T^{cut}$ does not invalidate the conditional independence properties of standard cuts in CEGs (\cite{Smith.Anderson.2008}, p. 55). The result therefore follows.

%%%%%%%%%%%%%%%%%%%%%%%% 
\section{Proof of Theorem \ref{the:fine_cut_independence}}
\label{app:fine_cut_independence}

\noindent
\emph{
	Take a fine cut $\mathcal{W}_t^{cut}$, $t=-1,0,1,\ldots$, in an $N$T-DCEG $\mathbb{C}$. For every ${s=0,1,\ldots}$, we have that
	\vspace{-9pt}
	\begin{equation*}
	\boldsymbol{X}(\mathcal{W}_t^{cut(s)})
	\independent
	\boldsymbol{Z}(\mathcal{W}_t^{cut(s)})
	|
	\boldsymbol{Q}(\mathcal{W}_t^{cut(s)}).
	%\label{eq:fine_cut_independence_proof}
	\vspace{-9pt}
	\end{equation*}
	Additionally, if a function $f({Z}(\mathcal{W}_t^{cut(s)}))$ satisfies
	\vspace{-9pt}
	\begin{equation*}
	{X}(\mathcal{W}_t^{cut(s)})
	\independent
	{Z}(\mathcal{W}_t^{cut(s)})
	|
	f({Z}(\mathcal{W}_t^{cut(s)})),
	%\label{eq:fine_cut_independence_function_proof}
	\vspace{-9pt}
	\end{equation*}
	then ${Q}(\mathcal{W}_t^{cut(s)})$ is a function of $f({Z}(\mathcal{W}_t^{cut(s)}))$  with probability one. These results also hold when a fine cut $\mathcal{W}_T^{cut(s)}$ is defined in a CEG $\mathbb{C}_{t+s} \in \mathcal{F}(\mathbb{C})$, ${t=T,T+1,\ldots}$.
	\vspace{3pt}}

We can assert immediately from the construction that given a value $q$ for $Q(\mathcal{W}_t^{cut(s)})$, then any random variable associated with~$\varpi_{Q(\mathcal{W}_t^{cut(s)})}(q)$ is completely defined. So none of the $w_0$-to-$\varpi_{Q(\mathcal{W}_t^{cut(s)})}(q)$ walks can bring any additional information on the realization of the random variable $X(\mathcal{W}_t^{cut(s)})$. 
Thus, $X(\mathcal{W}_t^{cut(s)}) \independent Q(\mathcal{W}_t^{cut(s)})|Z(\mathcal{W}_t^{cut(s)})$.

Now suppose that $Q(\mathcal{W}_t^{cut(s)})$ is not a function of $f(Z(\mathcal{W}_t^{cut}))$ with probability one. Then, for $s=N-1,N,\ldots$, there are at least two non-zero probability walks~$\lambda_1$ and $\lambda_2$ in $\Lambda(\mathcal{W}_t^{cut})$ such that $l(\lambda_1) \neq l(\lambda_2)$, $f(z_1)=f(z_2)$ and
\vspace{-5pt}
$$X(\mathcal{W}_t^{cut(s)})|[Z(\mathcal{W}_t^{cut(s)})=z_1] \not\equiv X(\mathcal{W}_t^{cut(s)})|[Z(\mathcal{W}_t^{cut(s)})=z_2],\vspace{-5pt}$$ 
where $z_1=\varpi_{Z(\mathcal{W}_t^{cut})}^{-1}(\lambda_1)$ and $z_2=\zeta_{Z(\mathcal{W}_t^{cut})}^{-1}(\lambda_2)$.
Thus, this would imply a contraction because $X(\mathcal{W}_t^{cut(s)})$ and $Z(\mathcal{W}_t^{cut(s)})$ were not conditionally independent given $f(Z(\mathcal{U}_t^{cut}))$. For time-horizon~$s$, $s=0,\ldots,N\!-\!2$, the proof is completely analogous to that one except that the condition  $l(\lambda_1) \neq l(\lambda_2)$ needs to be rewritten as follows:  $l(\lambda_1)$ and $l(\lambda_2)$ are in different set of the partition $\Xi_t^s$.  The result then follows.

From~\cite[Theorem~$5$]{Collazo.Smith.2018a} we can assert that in an $N\text{T-DCEG }$ $\mathbb{C}$ a fine cut $\mathcal{W}_t^{cut(s)}$  also defines a fine cut $\mathcal{W}_t^{cut(s)}$  at time $T$ in every CEG ${\mathbb{C}_{t+s} \subset \mathcal{F}(\mathbb{C})}$, ${t=T, T+1,\ldots}$, ${s=0,1,\ldots}$.  By construction, equations~\ref{eq:X_fine_cut_mass_function}, \ref{eq:Q_fine_cut_mass_function}, \ref{eq:Q_fine_cut_s_mass_function} and~\ref{eq:Z_fine_cut_mass_function} also hold. The result then follows due to the conditional independence properties of standard fine cuts in CEGs (\cite{Smith.Anderson.2008}, p. 61).

%%%%%%%%%%%%%%%%%%%%%%%% 
\section{Proof of Theorem \ref{pro:stochastic_local_independence}}
\label{app:local_stochast_simult_independence}

\noindent
\emph{
	Two random variables $\boldsymbol{X}$ and $\boldsymbol{Y}$ measurable with respect a DCEG are $T$-stochastically independent if and only if they are mutually $T\text{-locally}$ independent and $T$-contemporaneously independent.
	\vspace{3pt}}

Let $\boldsymbol{X}^{(t)}=(\boldsymbol{X}(0),\boldsymbol{X}(1),\ldots,\boldsymbol{X}(t))$. Assuming that the random vectors are $T$~{$\!\!\text{-}\!$}~stochastically independent, it then follows from equation~\ref{eq:stochastic_independence} that for every $t, t\geq T$,
\vspace{-5pt}
\begin{IEEEeqnarray}{rCl}
	p_{\boldsymbol{X}(t)}(\boldsymbol{x}(t)|\mathcal{E}^{(t-1)})
	&=&
	\sum_{\boldsymbol{y}{(t)}}
	p_{\boldsymbol{X}(t),\boldsymbol{Y}(t)}
	(\boldsymbol{x}(t),\boldsymbol{y}(t)|\mathcal{E}^{(t-1)})
	\nonumber \\
	&=&
	\sum_{\boldsymbol{y}{(t)}}
	p_{\boldsymbol{X}(t)}(\boldsymbol{x}(t)|\mathcal{E}^{(t-1)}_{(-\boldsymbol{Y})})	
	p_{\boldsymbol{Y}(t)}(\boldsymbol{y}(t)|\mathcal{E}^{(t-1)}_{(-\boldsymbol{X})})
	\nonumber \\
	&=&
	p_{\boldsymbol{X}(t)}(\boldsymbol{x}(t)|\mathcal{E}^{(t-1)}_{(-\boldsymbol{Y})}).
	\nonumber
\end{IEEEeqnarray}
Of course, we can obtain a completely analogous result for $\boldsymbol{Y}^{(t)}$. So these vectors are mutually $T$-locally independent. Substituting this result into equation~\ref{eq:stochastic_independence} it is straightforward to see that these vectors are also $T$-contempora-neously independent.

Conversely it is also true that
\vspace{-5pt}
\begin{IEEEeqnarray}{rCl}
	p_{\boldsymbol{X}(t),\boldsymbol{Y}(t)}
	(\boldsymbol{x}(t),\boldsymbol{y}(t)|\mathcal{E}^{(t-1)})
	&=&
	p_{\boldsymbol{X}(t)}(\boldsymbol{x}(t)|\mathcal{E}^{(t-1)})
	p_{\boldsymbol{Y}(t)}(\boldsymbol{y}(t)|\mathcal{E}^{(t-1)})	
	\nonumber \\
	&=&
	p_{\boldsymbol{X}(t)}(\boldsymbol{x}(t)|\mathcal{E}^{(t-1)}_{(-\boldsymbol{Y})})
	p_{\boldsymbol{Y}(t)}(\boldsymbol{y}(t)|\mathcal{E}^{(t-1)}_{(-\boldsymbol{X})}).
	\nonumber
\end{IEEEeqnarray}
Note that the first and second equalities follows, respectively, from the assumptions that the vectors $\boldsymbol{X}$ and $\boldsymbol{Y}$ are T-contemporaneously independent and mutually $T$-locally independent.

%%%%%%%%%%%%%%%%%%%%%%%%%%%%%%%%%%%%%%%%%%%%%%%%%%%%%%%%%%%%%%%
\section*{Acknowledgements}
\label{sec:Acknowledgement}
Rodrigo A. Collazo was supported by the Brazilian Navy and CNPq-Brazil [grant number 229058/2013-2]. Jim Q. Smith was supported by the Alan Turing Institute and funded by EPSRC [grant number EP/K039628/1].

%\singlespacing
\fontsize{11pt}{12pt}\selectfont 
%%%%%%%%%%%%%%%%%%%%%%%%%%%%%%%%%%%%%%%%%%%%%%%%%%%%%%%%%%%%%%%
%\section*{References}
\bibliographystyle{elsarticle-num} 
\bibliography{Bibliography_v1}

\end{document}